\documentclass[lettersize,journal]{IEEEtran}
\usepackage{amsmath,amsfonts}
\usepackage{algorithmic}
\usepackage{algorithm}
\usepackage{array}
\usepackage[caption=false,font=normalsize,labelfont=sf,textfont=sf]{subfig}
\usepackage{textcomp}
\usepackage{stfloats}
\usepackage{url}
\usepackage{verbatim}
\usepackage{graphicx}
\usepackage{cite}
\usepackage{booktabs}   
\usepackage{diagbox}
\usepackage{multirow}
\usepackage{amssymb}
\usepackage{tabularx}
\usepackage{makecell}
\usepackage{amsthm}
\newtheorem{theorem}{Theorem}

\newtheorem{lemma}{Lemma}
\newtheoremstyle{noDot}
  {3pt}      
  {3pt}      
  {\itshape} 
  {}         
  {\bfseries}
  {}         
  { }        
  {}         
\theoremstyle{noDot}
\newtheorem{remark}{Remark}
\newtheorem{assumption}{Assumption}
\usepackage{mathtools}
\usepackage{cite}
\usepackage{enumitem}
\usepackage[colorlinks=true, linkcolor=blue, citecolor=blue, urlcolor=blue]{hyperref}
\usepackage{ragged2e}
\usepackage{caption}
\begin{document}

\title{Safety-Regulated Transfer Reinforcement Learning with Adaptive Teacher Guidance}
\author{{Wenjie Huang, Yang Li, \textit{Member, IEEE}, Jingjia Teng, Mingwei Jin, Kai Song, \\
Zeyu Yang, Qisong Yang, Yougang Bian, \textit{Member, IEEE}}
\vspace{-10mm}
\thanks{Wenjie Huang, Yang Li, Jingjia Teng, Mingwei Jin, Kai Song, Zeyu Yang, Yougang Bian are with the State Key Laboratory of Advanced Design and Manufacturing Technology for Vehicle, College of Mechanical and Vehicle Engineering, Hunan University, Changsha 410082, China
.}
\thanks{Qisong Yang is with the Xi'an Institute of High-Tech, Xi'an 710025, China. 
}
}

\markboth{IEEE Internet of Things Journal}%
{Shell \MakeLowercase{\textit{et al.}}: A Sample Article Using IEEEtran.cls for IEEE Journals}
\maketitle

\begin{abstract}
Transfer reinforcement learning can improve policy learning efficiency by reusing knowledge from source tasks.
However, existing transfer reinforcement learning methods may suffer from negative transfer and unstable convergence due to source--target distribution mismatch, while the target-domain interactions required for policy adaptation can introduce substantial safety risks. To address these issues, we propose \underline{S}afety-\underline{R}egulated \underline{A}daptive \underline{T}ransfer \underline{R}einforcement \underline{L}earning (SRATRL), a teacher--student framework that combines safety-triggered intervention, safety-adaptive value shaping, and policy-compatibility-based optimization for efficient target-domain adaptation. First, a safety-triggered closed-loop intervention strategy is developed that activates teacher guidance according to the instantaneous safety cost and adaptively adjusts the intervention threshold based on the student policy's recent safety performance, thereby providing timely safety supervision while progressively restoring student autonomy as its safety improves. Next, a safety-adaptive teacher-guided value-shaping scheme is introduced, in which a teacher-consistency signal is incorporated into the critic target, and its contribution is dynamically regulated by the safety-constraint multiplier, enabling stronger teacher guidance under elevated safety risks and gradually weakening such guidance as the safety constraint is better satisfied. In addition, a teacher-student policy-compatibility weighting approach is proposed to alleviate the adverse optimization effects caused by policy mismatch. It reweights teacher-intervened transitions according to the relative likelihood of the executed action under the teacher and student policies, thereby improving policy-update stability. Moreover, return-deviation bounds are derived for the resulting mixed behavior policy to characterize the effect of teacher intervention. Experimental results demonstrate that compared with a Proximal Policy Optimization with Lagrangian constraint baseline, the proposed method improves the average velocity by 6.90\%, and reduces the crash ratio by 75.00\%. These results demonstrate that the proposed method can reduce safety costs while maintaining competitive task efficiency.
\end{abstract}


\vspace{-5mm}
\section{Introduction}
\IEEEPARstart{S}{afe} and efficient lane-change decision-making is essential for autonomous driving systems in complex traffic environments, and reinforcement learning (RL) has shown considerable potential for learning adaptive lane-changing policies through interaction with dynamic traffic scenarios \cite{Kiran2022DRLSurvey}.
However, learning a lane-changing policy from scratch typically requires extensive environment interactions and exploration, resulting in high training costs and considerable safety risks \cite{Krasowski2020SafeLaneChange}. To address these issues, transfer reinforcement learning leverages prior knowledge from source tasks to reduce the training burden and improve convergence efficiency \cite{Ma2024TITS,Shu2022TVT}. Nevertheless, existing transfer reinforcement learning methods still suffer from several limitations. First, distribution discrepancies between source and target domains often lead to knowledge mismatch, which may further result in poor data utilization, negative transfer, and degraded convergence performance \cite{Zhang2024PerspectiveO2O}. Second, although transfer mechanisms can improve learning efficiency, target-domain training still generally requires exploration, making it difficult to reduce unsafe exploratory interactions and limiting the deployment of such methods in safety-critical scenarios \cite{Hsu2023SimToLabToReal}. Therefore, existing transfer reinforcement learning methods for autonomous driving still face challenges in achieving a favorable trade-off between training safety and target-domain policy performance.

\begin{figure}[t]
    \centering
    \includegraphics[width=1\linewidth]{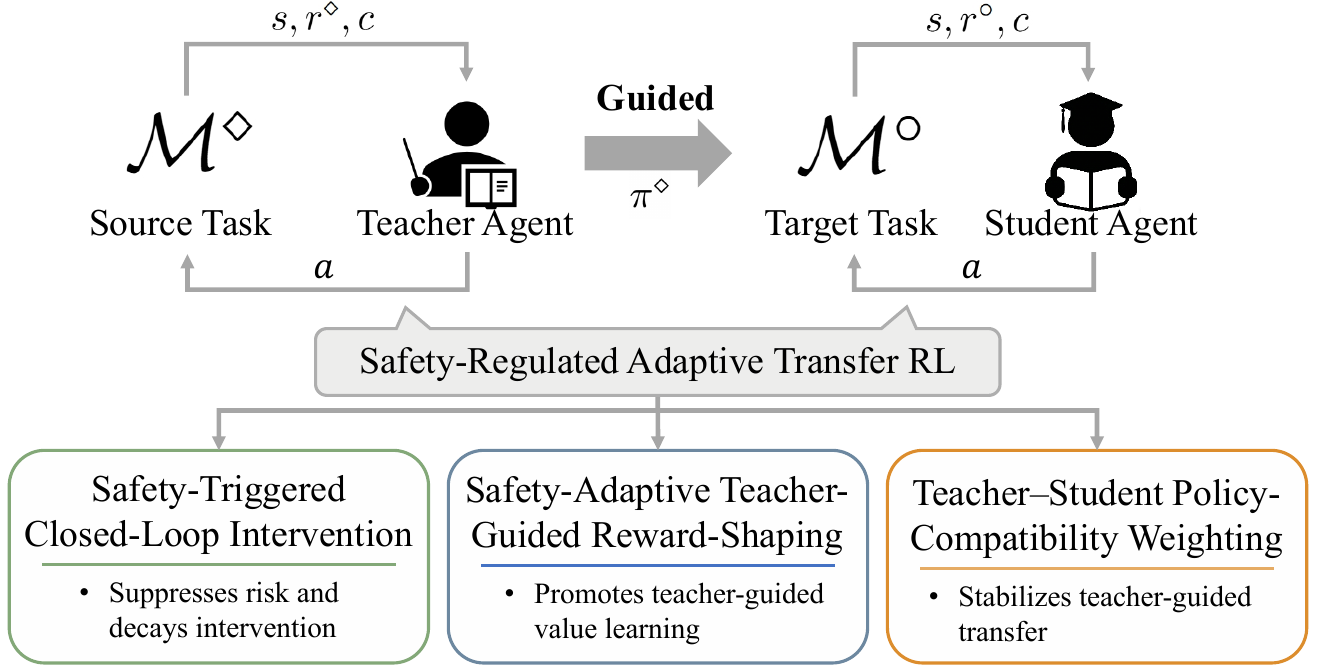}
    \caption{Illustration of the key components of the proposed method for transferring knowledge acquired from the source task ($\diamond$, left) to the target task ($\circ$, right). The proposed method is designed to accelerate learning and improve safety performance.}
    \label{fig-trasnfer-learning}
    \vspace{-5mm}
\end{figure}

To address the aforementioned challenges, existing transfer reinforcement learning methods can be broadly classified into three categories: safety-oriented transfer methods \cite{Zhou2025S2CD, Xu2018ZeroShotTransfer, Zhang2024SafetyControlTL}, efficiency-oriented transfer methods \cite{Liang2019FedTransferRL, Lu2024SceTL}, and domain adaptation-based transfer methods \cite{Moni2025SSDT, You2022CAT}.
Safety-oriented transfer methods focus on improving training safety during transfer, typically by reducing hazardous exploration in target-domain training through mechanisms such as teacher guidance \cite{Zhou2025S2CD}, robust control \cite{Xu2018ZeroShotTransfer}, and safety constraints \cite{Zhang2024SafetyControlTL}. 
However, unsafe behaviors may still arise during early adaptation to a new environment, and teacher guidance is often used mainly at the action-selection level, without fully transferring teacher-side safety preferences into the student’s value learning.
Efficiency-oriented transfer methods focus on improving transfer efficiency by reusing prior knowledge from source tasks to reduce the training burden and accelerate policy convergence \cite{Liang2019FedTransferRL, Lu2024SceTL}.
However, such knowledge is often used as a general prior, with limited consideration of its safety relevance to the current state and its compatibility with the evolving target policy. This issue becomes critical in teacher-guided safe transfer, where the teacher policy should be selectively used and properly weighted during target-task learning.
Domain adaptation-based transfer methods focus on modeling source-target discrepancies, typically through domain transfer \cite{Moni2025SSDT}, correspondence learning, or representation transfer to improve cross-domain transfer performance \cite{You2022CAT}. 
However, they are not specifically designed to address unsafe exploration or the safety-aware utilization of teacher policies during target-domain learning. 
Therefore, it remains imperative to develop a transfer mechanism that improves the convergence efficiency of policy learning while ensuring safety, so as to enable fast and safe reinforcement learning-based decision-making.

To address the above challenges, we propose Safety-Regulated Adaptive Transfer Reinforcement Learning (SRATRL) for autonomous highway lane changing, as illustrated in Fig. \ref{fig-trasnfer-learning}. SRATRL jointly regulates teacher influence at the action, value, and sample levels within a unified safety-regulated framework. Specifically, teacher involvement is adaptively controlled during interaction, teacher preferences are incorporated into value learning according to the current safety constraint pressure, and teacher-intervened samples are selectively weighted according to their relative compatibility with the evolving student policy. The main contributions are summarized as follows:

\begin{itemize}
    \item We design a safety-triggered closed-loop intervention strategy, which reduces high-cost exploration early in training and gradually relaxes teacher control as the student improves. The resulting mixed behavior policy is further analyzed through a return-deviation bound with respect to the teacher policy, providing theoretical insight into the effect of teacher intervention.
    \item We propose a safety-adaptive teacher-guided reward-shaping scheme that incorporates the teacher policy as a source-domain behavioral prior into student value learning through a teacher-induced shaping term. The influence of this term is dynamically modulated by the safety-related Lagrange multiplier, so that teacher guidance is strengthened when the student policy is unsafe and weakened as safety improves.
    \item We introduce a policy-compatibility weighting approach to adjust the contribution of teacher-intervened samples during optimization. By reweighting teacher-guided samples according to teacher-student action-likelihood compatibility, the proposed strategy suppresses the influence of mismatched samples.
    \item Comparative experiments conducted under different traffic densities, together with an additional NGSIM-based evaluation, demonstrate that the proposed method achieves a favorable balance between traffic efficiency and safety under the evaluated scenarios.
\end{itemize}

The remainder of this paper is organized as follows: Section \ref{section2} reviews related work. Section \ref{section3} describes the problem setting and formulation. Sections \ref{section4} and \ref{section6} present the proposed methodology and the theoretical analysis. Sections \ref{section7} and \ref{section8} provide the experimental setup and results. Section \ref{section9} concludes the whole paper.

\section{Related Work}
\label{section2}

\begin{figure*}[t]
    \centering
    \includegraphics[width=1\linewidth]{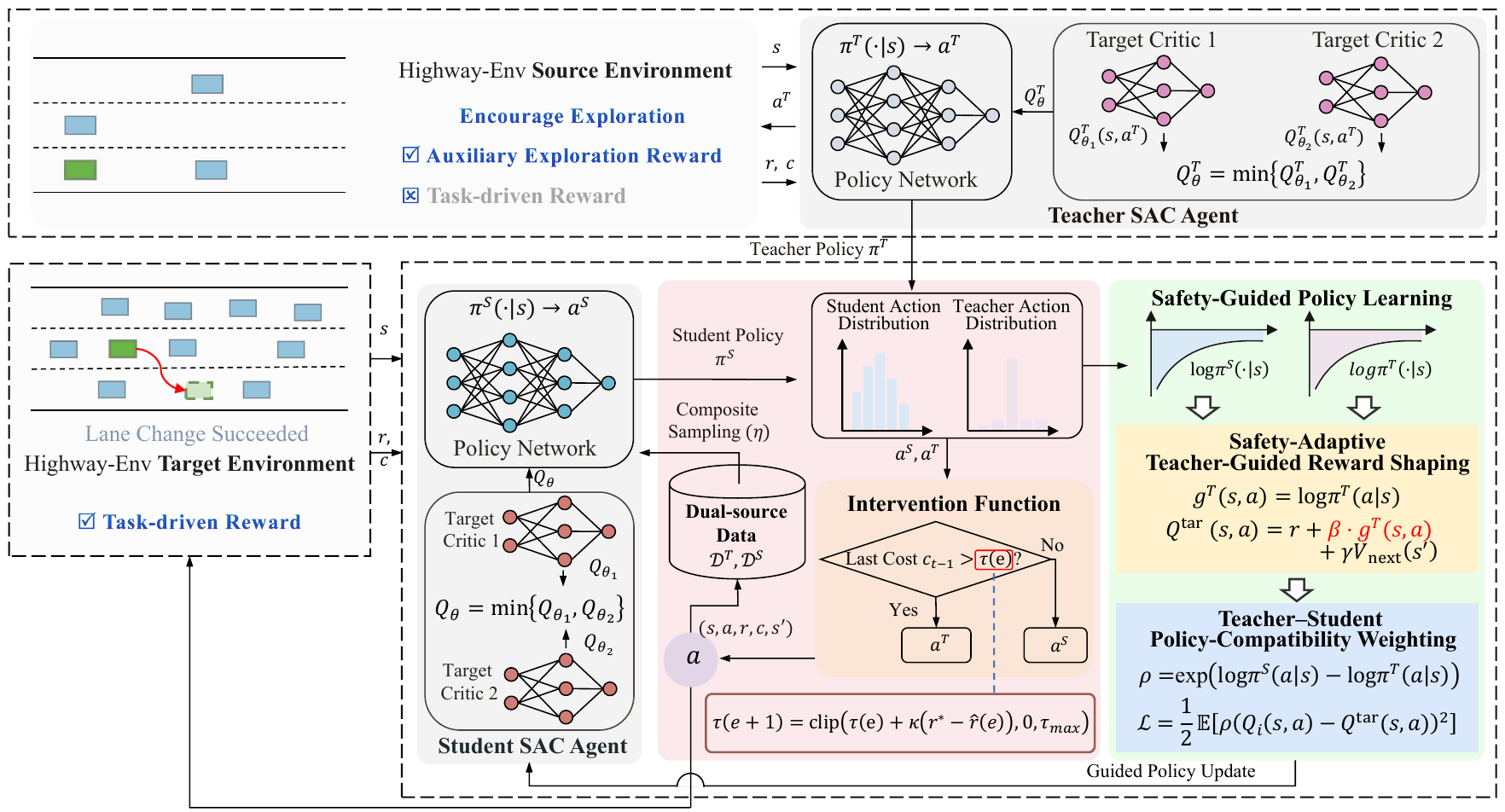}
    \caption{Framework of the proposed SRATRL method. A teacher SAC agent trained in a source-domain highway environment provides safety-oriented guidance for student adaptation in a target-domain environment through action intervention, teacher-guided reward shaping, policy-compatibility weighting, and dual-source replay learning. }
    \label{fig-framework}
    \vspace{-5mm}
\end{figure*}

\subsection{Transfer Reinforcement Learning for Autonomous Driving}
Transfer reinforcement learning improves adaptation to new tasks by reusing knowledge acquired from related source tasks. Existing approaches mainly transfer expert demonstrations, pretrained policies, model parameters, representations, or previously collected experience. Demonstration-based methods incorporate expert trajectories into policy optimization to accelerate early exploration \cite{Vecerik2017LeveragingDemonstrations,Kang2018POfD}, whereas policy transfer and distillation methods reuse pretrained behavioral or representational knowledge across related tasks \cite{Parisotto2016ActorMimic,Yin2017HierarchicalExperienceReplay}. Experience-based methods further improve learning efficiency by exploiting demonstrations or successful trajectories collected from related tasks \cite{Guo2023CRSfD,Filos2021SuccessorDemo}. In autonomous driving, transfer has also been explored across simulation and deployment domains, vehicle-dynamics variations, and different driving scenarios \cite{Mueller2018DrivingPolicyTransfer,Xu2018ZeroShotDrivingTransfer}. Online teacher-assisted transfer further allows a source policy to guide target-domain exploration through selective action advice \cite{Campbell2023IntrospectiveActionAdvising}.
However, existing studies mainly emphasize sample efficiency and faster convergence, while transferred guidance is often reused through fixed objectives or uniform weighting. In lane-changing tasks, such unregulated transfer may impose unsuitable source-domain preferences as the student’s safety condition evolves. Therefore, the influence of source knowledge should be adaptively regulated during target-domain learning.

Different from conventional transfer approaches, our approach incorporates teacher-side action evaluation into student value learning through safety-adaptive reward shaping, allowing transferred knowledge to be regulated according to the student’s safety condition rather than being reused uniformly.

\subsection{Safe Reinforcement Learning for Autonomous Lane-Changing}
Safe reinforcement learning aims to maximize performance while limiting unsafe interactions during training and execution. A common formulation is the constrained Markov decision process, where the expected safety cost must satisfy a predefined constraint \cite{Achiam2017CPO}. Lagrangian methods incorporate such constraints into policy optimization through an adaptive multiplier \cite{stooke2020responsive}, while risk-sensitive methods account for uncertainty, tail risk, or worst-case outcomes in the optimization objective \cite{Chow2015CVaR}. Another group of methods directly modifies unsafe actions. Safety shielding and action-filtering methods prevent or correct actions that violate predefined safety conditions \cite{Alshiekh2018Shielding,Dalal2018SafeExploration}, whereas backup controllers and intervention mechanisms replace hazardous actions with safer alternatives \cite{Thananjeyan2021RecoveryRL,Huang2024HAIMDRL}.
Although constrained optimization can reduce long-term safety costs, it does not necessarily prevent individual hazardous actions during early exploration, especially before reliable safety estimates have been learned \cite{stooke2020responsive}. Safety filters and external interventions can provide immediate protection, but fixed or overly conservative rules may restrict exploration and reduce policy autonomy. Frequent intervention may also cause the learned policy to depend excessively on the external controller.

In contrast, we introduce a safety-triggered closed-loop intervention strategy that suppresses unsafe exploratory interactions during early training and progressively decreases teacher involvement as the student develops safer target-domain behaviors.

\subsection{Teacher--Student Policy Guidance and Knowledge Transfer}
Teacher--student reinforcement learning transfers knowledge through policy distillation, imitation learning, action advising, teacher intervention, reward shaping, or the joint use of teacher-related experience \cite{Gupta2019AdviceReplay}. Policy distillation transfers behavioral or representational knowledge from a pretrained policy to a student policy \cite{Parisotto2016ActorMimic,Yin2017HierarchicalExperienceReplay}, while online advising methods allow the teacher to provide actions selectively during student interaction \cite{Wang2017ConfidenceBasedDemonstrations,Anand2021StateCategorization}. Some studies further combine teacher advice with imitation learning so that the student can reproduce teacher-supported behavior after guidance is removed \cite{Ilhan2021AdviceImitation}.
However, guidance from a source-domain teacher is not always beneficial \cite{Campbell2023IntrospectiveActionAdvising}. The teacher may provide actions that are safe but inefficient, or actions that are unsuitable for target-domain traffic conditions \cite{yang2021efficient}. Strong imitation may prevent the student from adapting beyond the teacher, whereas weak guidance may have little effect on learning. In addition, teacher-intervened samples are generated by a behavior policy different from the current student policy, creating an off-policy discrepancy whose influence may change as the student evolves.

Unlike existing teacher-guided methods, we use a likelihood-ratio-based policy-compatibility factor to regulate the contribution of teacher-intervened samples, thereby limiting the influence of samples that are highly mismatched with the current student policy.

\section{System Overview and Problem Description}
\label{section3}
\subsection{Overview of the Framework}
Fig. \ref{fig-framework} illustrates the overall architecture of the proposed teacher-guided safe transfer reinforcement learning framework. A teacher SAC agent is first trained in the source highway environment and then introduced into the target environment as a source-domain behavioral prior for student training. During online interaction, an adaptive intervention function determines whether the executed action is generated by the student policy or replaced by the teacher policy according to the recent safety cost and the current intervention threshold. This process reduces repeated high-cost exploration and naturally produces two types of transitions, namely student-generated samples and teacher-intervened samples, which are stored separately and jointly used through dual-source replay. For each transition, the teacher policy further provides an action-evaluation signal, which is incorporated into the value target as a teacher-induced shaping term to guide the student toward safer and more teacher-consistent behaviors. Meanwhile, a likelihood-ratio-based compatibility weight is computed to adjust the contribution of teacher-intervened samples during critic and actor updates, so that samples more compatible with the current student policy are emphasized in the optimization process. Through the joint use of adaptive intervention, teacher-induced value guidance, compatibility-aware reweighting, and dual-source replay, the student policy is progressively optimized to achieve a better balance between lane-changing efficiency and safety.

\subsection{Problem Formulation}
We consider tasks formulated by a constrained Markov decision process (CMDP) \cite{borkar2005actor,machines14060605}.
A CMDP can be described by a tuple $\mathcal{M} = \langle \mathcal{S}, \mathcal{A}, \mathcal{P}, r, c, d, \gamma \rangle$, where $\mathcal{S}$ denotes the state space and $\mathcal{A}$ denotes the action space. 
$\mathcal{P}:\mathcal{S}\times\mathcal{A}\times\mathcal{S}\rightarrow[0,1]$ specifies the transition probability from state $s$ to the next state $s'$ after taking action $a$, with $s\in\mathcal{S}$ and $a\in\mathcal{A}$. 
The function $r:\mathcal{S}\times\mathcal{A}\rightarrow\mathbb{R}$ denotes the reward function, and $c:\mathcal{S}\times\mathcal{A}\rightarrow\mathbb{R}$ denotes the cost function. 
$d\in\mathbb{R}_{+}$ and $\gamma\in[0,1)$ represent the cost limit and the discount factor, respectively.
In a safe RL problem, an agent interacts with the environment over time and generates a trajectory $\langle (s_0,a_0,r_0,c_0,s_1), (s_1,a_1,r_1,c_1,s_2), \ldots \rangle$. 
Starting from the initial state $s_0$, at each time step $t$, the agent is in a state $s_t \in \mathcal{S}$ and takes an action $a_t \in \mathcal{A}$. 
It then receives the corresponding reward $r_t = r(s_t,a_t)$ and cost $c_t = c(s_t,a_t)$, and transitions to the next state $s_{t+1} \sim \mathcal{P}(\cdot \mid s_t,a_t)$. This process continues from $s_{t+1}$ and repeats until a termination condition is met, after which a new trajectory is initiated.
The objective is to learn an optimal policy ($\pi^*$) that maximizes the expected discounted return while ensuring that the expected discounted cumulative cost remains below $d$:
\begin{equation}
\begin{aligned}
    \pi^*= \
    &\arg\max_{\pi}\; \underset{\xi\sim\rho_\pi}{\mathbb{E}} \left[ \sum_{t=0}^{\infty} \gamma^t r_t \right], \\
    &\text{s.t.} \quad \; \quad \ \ \underset{\xi\sim\rho_\pi}{\mathbb{E}} \left[ \sum_{t=0}^{\infty} \gamma^t c_t \right] \leq d,
\end{aligned}
\end{equation}
where $\xi$ denotes a trajectory sampled under the joint policy $\pi$, $\rho_{\pi}$ is the corresponding trajectory distribution.

We formalize our problem setting using the transfer learning framework for RL.
Given a source task and a target task, transfer learning aims to accelerate the learning process on the target task by leveraging knowledge acquired from the source task \cite{yang2023reinforcement}. 
In this study, a teacher policy ($\pi^T$) is first trained in the source task ($\mathcal{M}^\diamond$), and then introduced into the target task ($\mathcal{M}^\circ$) to guide the student agent through action intervention, thereby improving the efficiency of student policy learning.

\subsubsection{\textbf{State Space}}
A state describes the environment configuration with which the agent interacts. 
The observation of each vehicle $i$ is defined as $o_i=(p_i, \Delta x_i, \Delta y_i, \Delta v_i^{x}, \Delta v_i^{y})$, where $p_i$ is a presence indicator specifying whether vehicle $i$ is observable; 
$(\Delta x_i, \Delta y_i)$ denote the longitudinal and lateral positions of the surrounding vehicle $i$ relative to the ego vehicle, and $(\Delta v_i^{x}, \Delta v_i^{y})$ represent the relative longitudinal and lateral velocities between the surrounding vehicle $i$ and the ego vehicle, respectively.
The state space is defined as $\mathcal{S} = [o_1, o_2, \ldots, o_i, \ldots, o_M]$, and $M$ is  the maximum number of observed vehicles.

\subsubsection{\textbf{Action Space}}
The agent selects actions from the action space $\mathcal{A}$, which consists of high-level control commands. 
The action space $\mathcal{A}$ is defined as:
\begin{equation}
    \mathcal{A} = \left\{ A^{\leftarrow}, A^{\rightarrow}, A^{\uparrow}, A^{\downarrow}, A^{\oslash} \right\},
\end{equation}
where $A^{\leftarrow}$, $A^{\rightarrow}$, $A^{\uparrow}$, $A^{\downarrow}$, and $A^{\oslash}$ correspond to turning left, turning right, accelerating, decelerating, and maintaining speed, respectively.

\subsubsection{\textbf{Reward Function}}
The source and target tasks adopt different reward functions. Given the stringent safety requirements of safety-critical applications for transfer, we construct a source task without target-domain extrinsic task rewards to train a teacher policy that captures basic driving and exploration behaviors in a simpler environment. During training, no external task return is used; to improve exploration efficiency, an intrinsic auxiliary reward is introduced to encourage active motion and exploration:
\begin{equation}
r^\diamond = \left\| \mathbf{p}_{t}-\mathbf{p}_{t-1} \right\|_2 ,
\end{equation}
where $r^\diamond$ denotes the intrinsic auxiliary reward in the source task, and $\mathbf{p}_t = (x_t, y_t)$ denotes the position of the ego vehicle at time step $t$. 

In contrast, the target task follows a standard task-driven extrinsic reward mechanism, aiming to simultaneously promote safety, traffic efficiency, and task completion. In the highway lane-change scenario considered in this paper, the ego vehicle must complete a lane change toward the target lane and maintain stable driving within multi-lane traffic, while avoiding collisions, maintaining appropriate driving speed, and reaching the target lane, thereby enabling safe and efficient longitudinal progression and lateral decision-making. Accordingly, the reward received by the student agent at each time step is defined as a weighted sum of several components, including: a safety term that penalizes collision events, an efficiency term that encourages maintaining an appropriate driving speed, and a goal term that guides the policy toward the target lane.
Specifically, the instantaneous reward received is defined as:
\begin{equation}
r^\circ = w_s r_s + w_e r_e + w_g r_g,
\end{equation}
where $r^\circ$ denotes the instantaneous reward in the target task. $r_s$, $r_e$, and $r_g$ denote the collision penalty, efficiency reward, and goal reward, respectively. $w_s$, $w_e$, and $w_g$ are the corresponding weighting coefficients.

A collision penalty is imposed such that:
\begin{equation}
r_s =
\begin{cases}
-1, & \text{if a collision occurs}, \\
0,  & \text{otherwise},
\end{cases}
\end{equation}
which discourages unsafe behaviors that may cause collisions.

To promote driving efficiency, the speed reward is defined as:
\begin{equation}
r_e = \min \bigl(\frac{ v_{t} - v_{\min}}{ v_{\max} - v_{\min}},1 \bigr),
\end{equation}
where $v_{t}$ denotes the speed of the ego vehicle at time step $t$. 

To facilitate task completion in the highway lane-change scenario, a goal term is further introduced as:
\begin{equation}
r_g = 1 - \frac{\left|\ell_t-\ell_g\right|}{\mathcal{L}},
\end{equation}
where $\ell_t$ and $\ell_g$ denote the indices of the ego vehicle's current lane and the goal lane, respectively, and $\mathcal{L}$ is the total number of lanes on the current road. This normalized shaping reward increases as the ego vehicle approaches the target lane, thereby encouraging timely lane-change completion.

\subsubsection{\textbf{Cost Function}}
In addition to the reward design, we design a safety-related cost function to explicitly penalize situations with insufficient longitudinal safety margins. Specifically, we consider the time headway (THW) between the ego vehicle and the surrounding vehicle as the evaluation metric. The surrounding vehicle can be located either in front of or behind the ego vehicle in the current lane or the target lane. Let $d_\text{head}(t) \ge 0$ denote the longitudinal distance between the ego vehicle and surrounding vehicles at time step $t$. To penalize cases where the time headway falls below a predefined safety threshold $T_{\mathrm{safe}}$, the THW-based instantaneous cost is defined as:
\begin{equation}
c_{\mathrm{thw}}(t)=
\max \bigg(0,1-\dfrac{d_{\mathrm{head}}(t)}
{T_{\mathrm{safe}}\max(v_t,\varepsilon)}
\bigg),
\end{equation}
where $\varepsilon>0$ is a small constant to avoid numerical issues, and $T_{\mathrm{safe}}$ is the predefined safety threshold for time headway. This cost encourages the agent to maintain sufficient longitudinal interaction margins with surrounding vehicles during lane changing, thereby reducing close-distance driving risk. 



\section{Methodology}
\label{section4}
This section introduces the core components of the proposed SRATRL framework, including safety-regulated intervention, adaptive reward shaping, and policy-compatible optimization.

\subsection{Safety-Triggered Closed-Loop Teacher Intervention}
After completing the training of the teacher model, we train a new policy, referred to as the student policy, dedicated to the target task. To enhance safety during training and improve the student’s performance, we incorporate the teacher policy into the training loop, allowing the teacher policy $\pi^T$ and the student policy $\pi^S$ to work together and form the mixed behavior policy $\pi^{\mathrm{mix}}$. The term “teacher” only indicates that this policy is introduced to facilitate student training, and the optimality of the teacher model is not considered in this paper.

\textbf{Role of the Teacher Policy.} The teacher policy is used as a source-domain behavioral prior rather than an oracle safety expert. It is not assumed to be globally optimal in the target domain. Instead, the proposed framework selectively uses teacher guidance to reduce high-cost exploration during early target-domain adaptation, while gradually decreasing teacher influence as the student policy improves.

\subsubsection{\textbf{Action Intervention Function}}
During student training, the intervention function determines whether the action at the current time step is generated by the teacher policy network $\pi^T$ or the student policy network $\pi^S$. The student policy is then updated using the data collected under such interventions.
One intuitive design principle for the intervention function is to trigger intervention when the student’s behavior deviates from that of the teacher. In this work, we design a cost-based intervention function based on the instantaneous cost, as defined below:
\begin{equation}
    a_t = 
    \begin{cases} 
    a_t^T, & \text{if } c_{t-1} > \tau, \\
    a_t^S, & \text{otherwise,}
    \end{cases}
\end{equation}
where $\tau$ is the intervention threshold, which controls when the teacher policy intervenes. $a_t^T$ and $a_t^S$ denote teacher action and student action at the time step $t$, respectively.
The student policy keeps sampling with $\pi^{\mathrm{mix}}=\pi^{S}$ at the start of a trajectory. Once the first event $c_{t-1}>\tau$ is encountered, we set $\pi^{\mathrm{mix}}=\pi^{T}$ until the end of the episode. Consequently, the teacher policy provides a source-domain behavioral prior that can reduce high-cost exploration.
Based on this intervention function, the mixed behavior policy can be expressed in the following form:
\begin{equation}
\pi^{\text{mix}}(\cdot | s) = \mathcal{T}(s) \pi^T(\cdot | s) + (1 - \mathcal{T}(s)) \pi^S(\cdot | s),
\label{General representation of mixed behavior policy}
\end{equation}
where $\pi^{\mathrm{mix}}$ denotes the mixed behavior policy, and $\mathcal{T}(s) \in {0,1}$ denotes the intervention indicator at the current decision step. $\mathcal{T}(s)=1$ indicates that the teacher intervenes in action selection, whereas $\mathcal{T}(s)=0$ indicates otherwise.
The mixed behavior policy defined in \eqref{General representation of mixed behavior policy} provides an intuitive mechanism for safe exploration by allowing the teacher to intervene whenever necessary. A theoretical analysis of its effectiveness is presented in Section \ref{section6}.

\subsubsection{\textbf{Intervention Decay}}
Teacher intervention can reduce high-cost behaviors and accelerate the student’s learning process, substantially improving the student’s performance in the early stages of training. However, excessive reliance on the teacher’s guidance may hinder the student’s final performance. To address this issue, we introduce an intervention decay mechanism that gradually reduces the teacher’s intervention, thereby reducing the student’s dependence on the teacher. 
During training, this mechanism adaptively updates the intervention threshold $\tau$ based on recent safety performance. Specifically, the intervention rate is defined as the proportion of time steps within an epoch in which the teacher policy intervenes in the student policy. Let $r^*$ denote the reference intervention rate, and $\hat r(e)$ denote the average intervention rate at epoch $e$ computed using a sliding window over the most recent $W$ epochs. The threshold update rule is given by:
\begin{equation}
\tau(e+1)=\mathrm{clip}\Big(\tau(e)+\kappa\big(r^{*}-\hat r(e)\big),\,0,\,\tau_{\max}\Big),
\label{eq:tau_update}
\end{equation}
where $\kappa > 0$ is the update step size, $\tau_{\max}$ is the maximum allowable threshold, and $\mathrm{clip}(\cdot)$ ensures $\tau$ remains within a feasible range. The reference intervention rate $r^*$ is used to determine whether teacher intervention can be further relaxed.
As the student policy gradually becomes safer and the recent intervention rate falls below $r^*$, the intervention threshold increases, allowing the student to act more autonomously and progressively reducing its reliance on teacher intervention.

\subsubsection{\textbf{Dual-Source Data Training}}
The intervention-driven interaction naturally induces two types of transition samples with distinct origins. Let each transition be denoted by
\begin{equation}
\varsigma = (s, a, r, c, s').
\end{equation}
 According to whether teacher intervention is triggered, the collected transitions can be divided into a student-generated replay buffer and a teacher-generated replay buffer, which are defined as follows:
\begin{equation}
\mathcal{D}^S
=
\left\{
\varsigma \mid \mathcal{T}(s) = 0
\right\},
\qquad
\mathcal{D}^T
=
\left\{
\varsigma \mid \mathcal{T}(s) = 1
\right\},
\label{eq:dual_buffers}
\end{equation}
where $\mathcal{T}(s) \in \{0,1\}$ is the intervention indicator, with $\mathcal{T}(s)=1$ indicating teacher intervention and $\mathcal{T}(s)=0$ indicating student interaction.

Accordingly, the replay data used for training form a dual-source replay structure composed of $\mathcal{D}^S$ and $\mathcal{D}^T$, rather than a single-source experience distribution. Samples in $\mathcal{D}^T$ are mainly collected in safety-critical states and provide safety-aware guidance, whereas samples in $\mathcal{D}^S$ are generated during autonomous interaction and preserve task-driven exploration in nominal operating states. Therefore, the two replay sources play complementary roles in policy learning. To jointly utilize them, the mini-batch for each update is constructed as:
\begin{equation}
\mathcal{B}=\mathcal{B}^T \cup \mathcal{B}^S, \quad |\mathcal{B}^T| = \eta |\mathcal{B}|, \quad |\mathcal{B}^S| = (1 - \eta) |\mathcal{B}|,
\label{eq:batch_sampling}
\end{equation}
where $\mathcal{B}^{T} \subset \mathcal{D}^{T}$ and $\mathcal{B}^S \subset \mathcal{D}^S$ denote the teacher-generated and student-generated subsets contained in the current mini-batch, respectively, and $\eta$ is a mixing parameter determined by the teacher intervention rate during training.  Although $\mathcal{D}^{T}$ contains more samples from safety-critical states, the proposed dual-source replay does not overemphasize these risky regions. The sampling proportion is controlled by the teacher intervention rate and gradually decreases as teacher intervention decays. Consequently, student-generated samples from normal low-cost regions dominate the policy update in the later training stage, which helps preserve the generalization ability of the student policy in normal driving states.


\subsection{Safety-Adaptive Teacher-Guided Value Shaping}

\begin{figure}[t]
    \centering
    \includegraphics[width=0.95\linewidth]{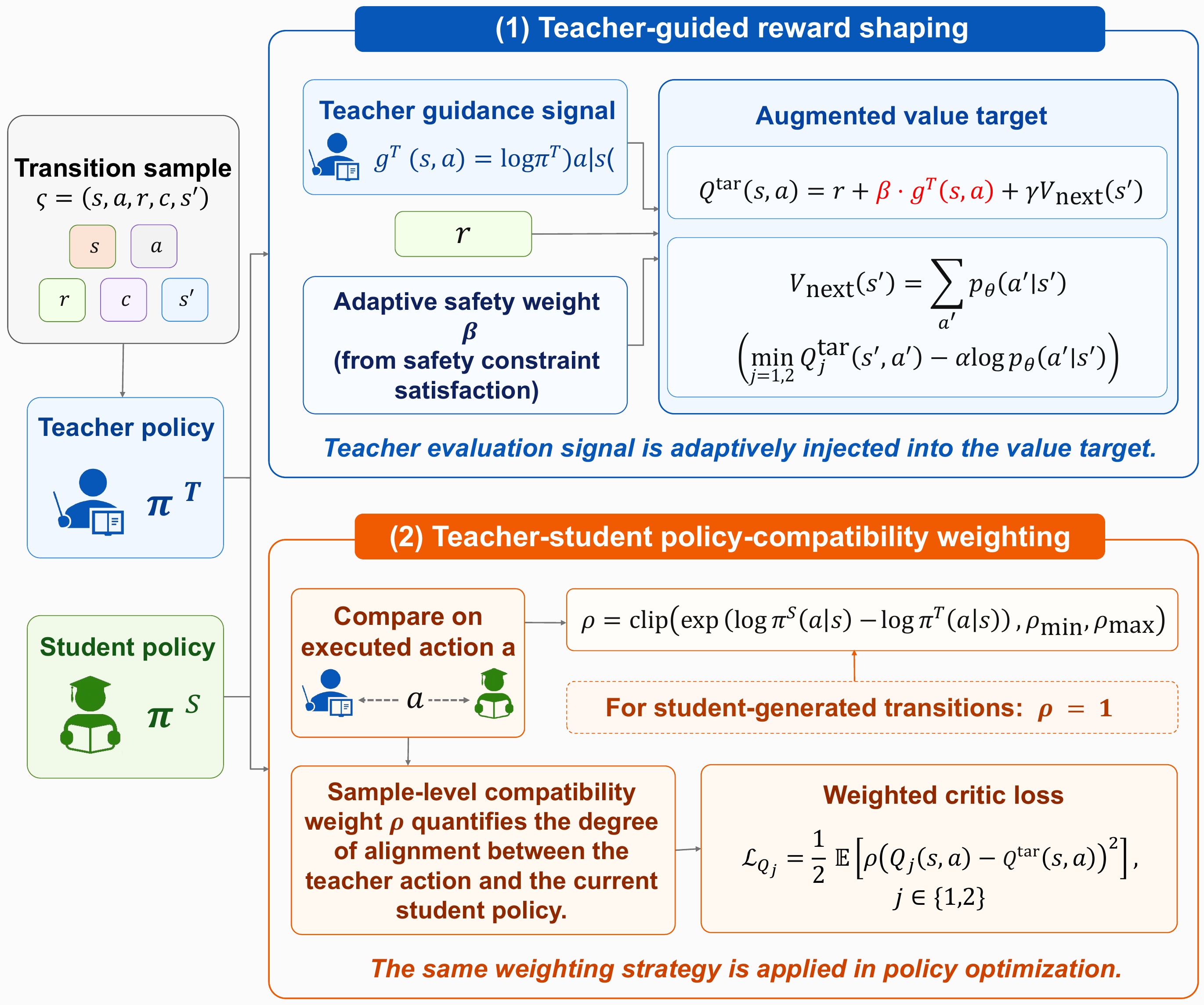}
    \vspace{0.2cm}
    \caption{The proposed framework incorporates teacher guidance into the student policy optimization process to improve learning safety and transfer stability under safety constraints.}
    \label{fig-policy-learning}
    \vspace{-5mm}
\end{figure}

\subsubsection{\textbf{Teacher-induced Reward Shaping}}
For any transition $( s, a, r, c, s')$, in addition to the extrinsic environment reward $r$, we compute an auxiliary signal from the teacher policy to quantify how well the action $a$ aligns with the teacher’s behavior:
\begin{equation}
g^T(s,a) = \log \pi^{T}(a \mid s).
\end{equation}
Since $g^T(s,a)\leq 0$, this signal can be interpreted as a teacher-consistency regularization term. An action assigned a high probability by the teacher incurs only a small shaping penalty. In this way, the teacher prior biases value learning toward teacher-consistent behaviors during target-domain adaptation.
To adaptively regulate the strength of this guidance according to the current safety condition, the teacher-consistency signal is scaled by $\beta$, which is the Lagrange multiplier associated with the safety-cost constraint. As the safety-constraint pressure increases, $\beta$ increases and strengthens the penalty on teacher-inconsistent actions. Conversely, as the safety constraint is better satisfied, the influence of the teacher-consistency term is gradually reduced. This shared multiplier directly couples the strength of transferred teacher guidance with the current safety-constraint pressure.

\subsubsection{\textbf{Value Target Construction}}
Within the maximum-entropy actor--critic framework, the update of the action-value function is typically based on a target value obtained by marginalizing over the next-state action distribution, thereby incorporating entropy regularization into value estimation. Specifically, given the next state $s'$, the policy $\pi_\theta(\cdot\mid s')$ outputs the action probabilities $p_\theta(\cdot\mid s')$ and the corresponding log-likelihoods $\log p_\theta(\cdot\mid s')$. The expected next-state value is then computed as:
\begin{equation}
V_{\text{next}}(s')
=
\sum_{a'} p_\theta(a'\mid s')
(
\min_{j=1,2} Q_j^{\mathrm{tar}}(s',a')
-\alpha \log p_\theta(a'\mid s')
),
\end{equation}
where $\alpha$ is the entropy regularization coefficient, and $Q_j^{\mathrm{tar}}$ denotes the target action-value function parameterized by a slowly updated target network to stabilize learning.

To incorporate teacher-guidance information into value learning, we take the teacher guidance signal corresponding to the current action as an auxiliary term to construct an augmented target (see Fig. \ref{fig-policy-learning}):
\begin{equation}
Q^{\mathrm{tar}}(s,a) = r + \beta \cdot g^T(s,a) + \gamma V_{\text{next}}(s'),
\label{eq-q-target}
\end{equation}
where $\gamma$ is the discount factor, and $Q^{\mathrm{tar}}(s,a)$ denotes the temporal-difference target used for critic learning. Compared with the standard maximum-entropy update, the additional term introduces a safety-adaptive teacher prior into value estimation. In this manner, the strength of teacher guidance is adaptively regulated according to the evolving safety-constraint pressure, rather than being applied with a fixed shaping strength.

\subsection{Teacher--Student Policy-Compatibility Weighting}
Besides incorporating the teacher’s evaluation signal into the value target, we further adjust the contribution of teacher-intervened transitions during optimization by introducing a likelihood-ratio-based weighting factor, as shown in Fig. \ref{fig-policy-learning}. Specifically, for a teacher-intervened action $a$, we define:
\begin{equation}
    \rho = \mathrm{clip} \left( \exp(\log \pi^S(a|s)-\log \pi^T(a|s)),\rho_\mathrm{min}, \rho_\mathrm{max} \right),
\end{equation}
where $\pi^S$ and $\pi^T$ denote the student and teacher policies, respectively. For student-generated transitions, we simply set $\rho=1$.
The ratio reflects the relative likelihood of the executed action under the current student policy rather than the global similarity between the teacher and student policy distributions. For a teacher-intervened transition, a larger value indicates that the executed action is relatively better supported by the evolving student policy, whereas a smaller value indicates a stronger discrepancy between the behavior that generated the sample and the current student policy. Therefore, the clipped ratio is used as a sample-level weighting factor rather than as an unbiased importance-sampling correction.
Accordingly, the critic network losses can be written in the weighted form:
\begin{equation}
    \mathcal{L}_{Q_j} = \frac{1}{2} \, \mathbb{E} \left[ \rho \big( Q_j(s, a) - Q^{\text{tar}}(s,a) \big)^2 \right], \quad j \in \{1, 2\},
    \label{eq-critic-loss}
\end{equation}
where $\mathcal{L}_{Q_j}$ is the critic loss, and $Q_j(s, a)$ denotes the action-value estimated by the corresponding online critic network. The same weighting strategy is also applied in the policy optimization step.


\subsection{SRATRL Algorithm}
\begin{algorithm}[t]
\caption{Safety-Regulated Adaptive Transfer RL}
\label{alg:tgst_sac}
\begin{algorithmic}[1]
\REQUIRE Student environment $\mathcal{E}^S$, teacher policy $\pi^T$, replay buffers $\mathcal{D}^S$, $\mathcal{D}^T$.
\STATE \textbf{Initialize} student policy network, reward and cost critic networks, corresponding critic target networks, intervention threshold $\tau$, and intervention flag $F_\text{int}$.
\FOR{each interaction step}
    \STATE Obtain current student state $s$.
    \IF{$F_\text{int}$= True}
        \STATE Select action $a \sim \pi^T(\cdot|s)$.
        \STATE Compute $\rho$ and teacher guidance signal $g^T(s,a)$.
    \ELSE
        \STATE Select action $a \sim \pi^S(\cdot|s)$.
        \STATE Set $\rho \gets 1$, and compute $g^T(s,a)$.
    \ENDIF

    \STATE Execute $a$ in $\mathcal{E}^S$, observe $(s,a,r,c,s')$.
    \IF{$F_\text{int} = \text{False}$}
        \STATE Store $(s,a,r,g^T(s,a),\rho,s',c)$ into $\mathcal{D}^S$.
    \ELSE
        \STATE Store $(s,a,r,g^T(s,a),\rho,s',c)$ into $\mathcal{D}^T$.
    \ENDIF
    \IF{$c > \tau$}
        \STATE Activate teacher intervention for the remaining episode: $F_\text{int}= \text{True}$.
    \ENDIF

    \STATE Set $s \gets s'$.

    \IF{end of episode}
        \STATE $F_\text{int}= \text{False}$.
    \ENDIF
    \IF{update condition is satisfied}
        \STATE Composite sampling from $\mathcal{D}^S$ and $\mathcal{D}^T$ with $\eta$.
        \STATE Update the reward and cost critics weighted by $\rho$.
        \STATE Update $\pi^{\mathrm{S}}$ by minimizing the actor loss.
    \ENDIF

\ENDFOR
\end{algorithmic}
\vspace{-1mm}
\end{algorithm}

The overall training procedure of SRATRL is summarized in Algorithm \ref{alg:tgst_sac}. During interaction, the student acts autonomously until the safety cost exceeds the intervention threshold, after which the teacher takes over for the remainder of the episode. Student-generated and teacher-intervened transitions are stored separately in $\mathcal{D}^S$ and $\mathcal{D}^T$, respectively. Teacher guidance is incorporated into learning through the teacher-consistency shaping signal $g^T(s,a)$ and the relative-likelihood weight $\rho$. During optimization, samples from the two replay buffers are jointly drawn according to the mixing ratio $\eta$, and the reward critics, cost critic, and student policy are updated accordingly. At the end of each epoch, the recent intervention rate is used to adapt the intervention threshold $\tau$ and replay mixing ratio $\eta$. Through this procedure, SRATRL coordinates teacher intervention, safety-adaptive value guidance, and weighted dual-source learning during target-domain adaptation.
\section{Theoretical Analysis}
\label{section6}
This section analyzes the performance deviation of the mixed behavior policy induced by teacher intervention. The analysis does not claim monotonic policy improvement or formal satisfaction of CMDP safety constraints, but instead focuses on the return deviation between the mixed behavior policy and the teacher policy. The results show that effective intervention constrains the mixed policy within a bounded neighborhood around the teacher policy. 

\begin{lemma}
\label{lamme1}
    Consider a MDP with discount factor $\gamma \in (0,1)$ and bounded reward $r(s,a) \in [0, R_{max}]$. For any stochastic policy $\pi$, define the discounted state distribution of $\pi$ as:
    \begin{equation}
    d_{\pi}(s) = 
    (1-\gamma)\sum_{t=0}^{\infty}\gamma^{t}\Pr\!\left(s_t=s\,\middle|\,\pi,d_0\right),
    \end{equation}
    where $d_0$ is the initial-state distribution and $\Pr(s_t=s\mid\pi,d_0)$ denotes the state visitation probability at time step $t$ under policy $\pi$.

    Then the following bound holds \cite{xue2023guarded}: 
    \begin{align}
        \left| J(\pi) - J(\pi') \right|
    \leq
    \frac{R_{\max}}{(1 - \gamma)^2}
    \, \mathbb{E}_{s \sim d_{\pi}}
    \left\| \pi(\cdot \mid s) - \pi'(\cdot \mid s) \right\|_1.
    \end{align}
\end{lemma}

\begin{proof}
    According to the Performance Difference Lemma \cite{kakade2002approximately}, for any two policies $\pi$ and $\pi'$, the difference in their discounted returns can be expressed as follows:
    \begin{equation}
J(\pi) - J(\pi') = \frac{1}{1 - \gamma} \, \mathbb{E}_{s \sim d_{\pi}} \left[ \mathbb{E}_{a \sim \pi(\cdot|s)} A_{\pi'}(s,a) \right],
\label{lamme1_p1}
\end{equation}
where $A_{\pi'}(s,a) \triangleq Q_{\pi'}(s,a) - V_{\pi'}(s)$ denotes the advantage function with respect to $\pi'$.

For a fixed state $s$,
\begin{equation}
\mathbb{E}_{a \sim \pi(\cdot|s)} A_{\pi'}(s,a) = \sum_a \pi(a|s) A_{\pi'}(s,a).
\end{equation}

From the definition of the advantage function, we have:
\begin{equation}
\mathbb{E}_{a \sim \pi'(\cdot|s)} A_{\pi'}(s,a) = 0.
\label{lamme1_p3}
\end{equation}

Therefore,
\begin{equation}
\begin{aligned}
&\mathbb{E}_{a \sim \pi(\cdot\mid s)} A_{\pi'}(s,a)
-\mathbb{E}_{a \sim \pi'(\cdot\mid s)} A_{\pi'}(s,a)\\
=& \sum_a \bigl(\pi(a\mid s)-\pi'(a\mid s)\bigr)\,A_{\pi'}(s,a),
\end{aligned}
\end{equation}
and using \eqref{lamme1_p3} we obtain:

\begin{equation}
\mathbb{E}_{a \sim \pi(\cdot|s)} A_{\pi'}(s,a) = \sum_a \bigl(\pi(a|s) - \pi'(a|s)\bigr) A_{\pi'}(s,a).
\label{lamme1_p5}
\end{equation}

Applying Hölder's inequality \cite{folland1999real} to \eqref{lamme1_p5} yields
\begin{equation}
\left| \mathbb{E}_{a \sim \pi(\cdot|s)} A_{\pi'}(s,a) \right|
\leq \|\pi(\cdot|s) - \pi'(\cdot|s)\|_1 \cdot \|A_{\pi'}(s, \cdot)\|_\infty.
\label{lamme1_p6}
\end{equation}

Because $ r(s,a) \leq R_{\max}$, for any policy $\pi'$ we have:
\begin{equation}
Q_{\pi'}(s,a) \leq \sum_{t=0}^{\infty} \gamma^t R_{\max} = \frac{R_{\max}}{1-\gamma}, \quad
V_{\pi'}(s) \leq \frac{R_{\max}}{1-\gamma}.
\end{equation}

Hence,
\begin{equation}
|A_{\pi'}(s,a)| = |Q_{\pi'}(s,a) - V_{\pi'}(s)|
\leq \frac{R_{\max}}{1-\gamma},
\label{lamme1_p8}
\end{equation}
which implies $\|A_{\pi'}(s, \cdot)\|_\infty \leq \frac{R_{\max}}{1-\gamma}$.

Substituting \eqref{lamme1_p6} and \eqref{lamme1_p8} into \eqref{lamme1_p1} gives
\begin{equation}
\begin{aligned}
|J(\pi) - J(\pi')|
&\leq \frac{1}{1-\gamma} \, \mathbb{E}_{s \sim d_{\pi}} 
\left[
\|\pi(\cdot|s) - \pi'(\cdot|s)\|_1 \cdot \frac{R_{\max}}{1-\gamma}
\right] \\
&= \frac{R_{\max}}{(1-\gamma)^2} \, \mathbb{E}_{s \sim d_{\pi}} 
\|\pi(\cdot|s) - \pi'(\cdot|s)\|_1,
\end{aligned}
\end{equation}
which is exactly Lemma \ref{lamme1}.
\end{proof}

\begin{assumption}[Policy Smoothness in Finite Discrete Action Space]
\label{assumption1}
Consider a finite discrete action space $\mathcal{A}$ with $|\mathcal{A}| < \infty$. The student policy $\pi^S$ is assumed to be uniformly lower bounded, i.e., there exists a constant $\delta > 0$ such that for all states $s \in \mathcal{S}$ and all actions $a \in \mathcal{A}$,

\begin{equation}
\pi^S(a \mid s) \ge \delta.
\end{equation}

\end{assumption}

\begin{theorem}
\label{theorem1}
With the action intervention function incorporated, the lower and upper bounds of the return of the behavior policy $J(\pi^{\mathrm{mix}})$ are respectively expressed as:

\begin{equation}
\begin{aligned}
&J(\pi^T) - \frac{\sqrt{2}(1-\omega)R_{\max}\sqrt{-\log \delta}}{(1-\gamma)^2} \\
\leq &J(\pi^{\mathrm{mix}}) \\
\leq &J(\pi^T) 
+ \frac{\sqrt{2}(1-\omega)R_{\max}\sqrt{-\log \delta}}{(1-\gamma)^2},
\end{aligned}
\end{equation}
where $\omega = \frac{\mathbb{E}_{s \sim d_{\pi^{\mathrm{mix}}}} \left\| \mathcal{T}(s) \left[ \pi^T(\cdot|s) - \pi^S(\cdot|s) \right] \right\|_1}{\mathbb{E}_{s \sim d_{\pi^{\mathrm{mix}}}} \left\| \pi^T(\cdot|s) - \pi^S(\cdot|s) \right\|_1}$ is the effective intervention coefficient determined by the switch function \cite{Zhou2025S2CD}.
\end{theorem}

\begin{proof}
Using the mixed behavior policy as defined in \eqref{General representation of mixed behavior policy},
\begin{equation}
\label{theorem1_p1}
\begin{aligned}
&\mathbb{E}_{s \sim d_{\text{mix}}} \left\| \pi^{\text{mix}}(\cdot|s) - \pi^T(\cdot|s) \right\|_1 \\
&= (1 - \omega) \mathbb{E}_{s \sim d_{\text{mix}}} \left\| \pi^S(\cdot|s) - \pi^T(\cdot|s) \right\|_1 .
\end{aligned}
\end{equation}


Applying Pinsker's inequality \cite{cover2006elements}, we have:
\begin{equation}
    \left\| \pi^S(\cdot|s) - \pi^T(\cdot|s) \right\|_1 
    \leq \sqrt{2 D_{\text{KL}}\bigl(\pi^T(\cdot|s) \parallel \pi^S(\cdot|s)\bigr)}.
\label{theorem1_p2}
\end{equation}

For any fixed state $s$, the Kullback-Leibler (KL) divergence between teacher policy and student policy is defined as:
\begin{equation}
    D_{\text{KL}}\bigl(\pi^T(\cdot|s) \parallel \pi^S(\cdot|s)\bigr) = \sum_a \pi^T(a|s) \log \frac{\pi^T(a|s)}{\pi^S(a|s)}.
\end{equation}

Since \(\pi^S(a|s) \ge \delta\), it follows that \(\frac{1}{\pi^S(a|s)} \leq \frac{1}{\delta}\). Therefore:

\begin{equation}
\begin{aligned}
&D_{\text{KL}}\bigl(\pi^T(\cdot|s) \parallel \pi^S(\cdot|s)\bigr) 
\leq \sum_a \pi^T(a\mid s) \log \frac{\pi^T(a\mid s)}{\delta} \\
&= \sum_a \pi^T(a\mid s) \log \pi^T(a\mid s)
   - \log \delta \sum_a \pi^T(a\mid s).
\end{aligned}
\end{equation}

Using the definition of entropy $H\bigl(\pi^T(\cdot|s)\bigr) = -\sum_a \pi^T(a|s) \log \pi^T(a|s) \ge 0$, we obtain:

\begin{equation}
D_{\text{KL}}\bigl(\pi^T(\cdot|s) \parallel \pi^S(\cdot|s)\bigr)
\leq -\log \delta.
\label{theorem1_p3}
\end{equation}

By combining (\ref{theorem1_p1}), (\ref{theorem1_p2}), (\ref{theorem1_p3}) with Lemma \ref{lamme1}, the following result can be obtained:
\begin{equation}
\begin{aligned}
\left| J(\pi^{\text{mix}}) - J(\pi^T) \right| 
\leq \frac{\sqrt{2}(1 - \omega) R_{\max} \sqrt{-\log \delta}}{(1 - \gamma)^2}.
\end{aligned}
\end{equation}

Therefore, we obtain:
\begin{equation}
\begin{aligned}
&J(\pi^T) - \frac{\sqrt{2}(1-\omega)R_{\max}\sqrt{-\log \delta}}{(1-\gamma)^2} \\
\leq &J(\pi^{\mathrm{mix}}) \\
\leq &J(\pi^T) 
+ \frac{\sqrt{2}(1-\omega)R_{\max}\sqrt{-\log \delta}}{(1-\gamma)^2}.
\end{aligned}
\end{equation}
\end{proof}

\begin{remark}
Theorem \ref{theorem1} indicates that stronger effective teacher involvement leads to a tighter bound on the return deviation between the mixed behavior policy and the teacher policy. Therefore, when the teacher provides a useful source-domain behavioral prior, intervention can constrain early target-domain behavior within a bounded return neighborhood of the teacher policy. This result characterizes the effect of teacher intervention rather than guaranteeing monotonic improvement or strict safety constraint satisfaction.
\end{remark}

\section{Experimental Setup}
\label{section7}
\subsection{Simulation Environment}
The proposed method is evaluated in a highway lane-changing scenario. In this setting, the ego vehicle makes lane-changing decisions under the dynamic interactions and influences of surrounding vehicles, with the objective of safely completing a lane change toward the rightmost target lane and maintaining stable driving. At the beginning of each episode, the ego vehicle and surrounding vehicles are randomly initialized on the road. The simulation parameters are set as follows: the lane width is 5 m; the vehicle length and width are 5 m and 2 m, respectively. If the ego vehicle makes an improper lane-change decision during driving that results in a collision or departure from the road boundary, the episode is deemed a failure, and the environment is immediately reset to start the next episode.
To evaluate the policy performance under different traffic conditions, we consider low-, medium-, and high-density traffic scenarios. The traffic density is adjusted by varying the longitudinal spacing between surrounding vehicles, defined as $d_{k,k+1}=d_s+v_{k+1}/\zeta$, where $d_{k,k+1}$ denotes the longitudinal spacing between adjacent vehicles, $v_{k+1}$ is the velocity of the rear vehicle, $d_s$ is the safety distance, and $\eta$ is the traffic-density coefficient. A larger $\zeta$ results in a smaller vehicle spacing and thus corresponds to a higher traffic density. Detailed traffic-density settings can be found in \cite{li2026risk}.

\subsection{Implementation Details}
The simulation environment is developed based on a modified version of the open-source \textit{gym-highway-env} simulator \cite{Edouard2018Highway}. Vehicle motion is propagated using a kinematic bicycle model \cite{Polack2017IV}. For HDVs, longitudinal dynamics and lateral lane-changing behaviors are governed by the Intelligent Driver Model (IDM) \cite{treiber2000congested} and the MOBIL model \cite{2001Preferred}, respectively.
The high-level decisions generated by the policy are executed in the simulation environment with a simulation frequency of 10 Hz.
To ensure reproducibility and statistical robustness, all algorithms are trained by interacting with the environment for $5 \times 10^5$ steps using multiple different random seeds. After training, all algorithms are evaluated using multiple random seeds with 100 test episodes per seed. The discount factor and learning rate are set to $0.99$ and $1 \times 10^{-4}$, respectively. Parameter gradients are updated every 100 steps, and each update samples a mini-batch of 256 transitions from the replay buffer.
Further details are shown in Table \ref{tab-Hyperparameters}.
All experiments are conducted on a personal PC equipped with an Intel Core i5-10210U CPU and 16 GB RAM
\footnote{\url{https://github.com/HuangWJ-12/TG-STRL}}.

\begin{table}[t]
    \caption{\textsc{Hyperparameters Setting}}
    \centering
    \begin{tabularx}{\linewidth}{l X c}
    \toprule
        Parameter & Description & Value \\
    \midrule
        $T_\text{safe}$ & Safe threshold for time headway & 1.2 s \\
        $\eta_{\pi}$ & Actor network learning rate & $1\times 10^\mathrm{-4}$ \\
        $\eta_{Q}$ & Reward critic network learning rate & $1\times 10^\mathrm{-4}$ \\ 
        $\eta_{C}$ & Cost critic network learning rate & $1\times 10^\mathrm{-4}$ \\ 
        $\gamma$ & Discount factor & 0.99 \\
        $\mathcal{B}$ & Mini-batch size & 256 \\
        $n_e$ & Total steps per epoch & 4000 \\
        $\kappa$ & Intervention threshold update step size & 0.5 \\
    \bottomrule
    \end{tabularx}
    \vspace{-5mm}
    \label{tab-Hyperparameters}
\end{table}

We compare the proposed method with representative baselines covering constrained reinforcement learning, direct policy transfer, and teacher-guided reinforcement learning. SAC-Lag \cite{yang2023safety} and PPO-Lag \cite{stooke2020responsive} are adopted as off-policy and on-policy constrained RL baselines, respectively, both incorporating Lagrangian safety optimization. In addition, Fine-tuning (FT) \cite{julian2021never} is considered as a direct policy-transfer baseline, where the source-domain policy is used to initialize the student policy for target-domain adaptation.
Expert Guided Policy Optimization (EPGO) \cite{peng2022safe} is included as a teacher-guided reinforcement learning baseline that employs expert intervention during online training. Together, these baselines allow us to evaluate the proposed method against safety-constrained policy learning, direct source-policy reuse, and action-level teacher guidance.
Considering safety and traffic efficiency, the following evaluation metrics are adopted: (i) Average Reward $(\uparrow)$: the average reward obtained by the ego vehicle per episode; (ii) Average Cost $(\downarrow)$: the average cost of the ego vehicle for each episode; (iii) Crash Ratio $(\downarrow)$: the ratio of the ego vehicle that collides with surrounding vehicles, represented as a decimal value between 0 and 1; (iv) Average Velocity $(\uparrow)$: the average speed maintained by the vehicle during driving.

\section{Experiment Results}
\label{section8}
\begin{figure*}[t]
    \centering
    \includegraphics[width=1.0\linewidth]{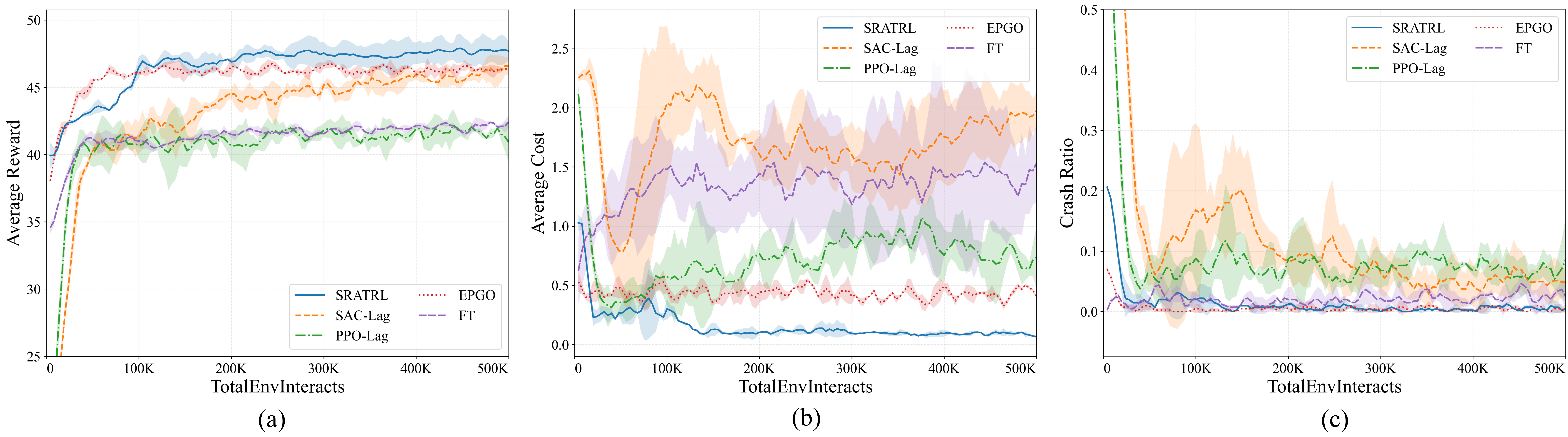}
    \vspace{-6mm}
    \caption{Comparative study. (a) Average reward, (b) Average cost, (c) Crash ratio. The shaded areas represent the standard deviation across multiple random seeds. Training curves of average reward, average cost, and crash ratio versus training steps for different methods in the target-task environment. The compared methods include SRATRL, SAC-Lag \cite{yang2023safety}, PPO-Lag \cite{stooke2020responsive}, EPGO~\cite{peng2022safe}, and FT~\cite{julian2021never}. The results show that SRATRL achieves higher reward with lower safety cost and crash ratio.}
    \label{fig-comparative}
    \vspace{-5mm}
\end{figure*}
\begin{table*}[t]
\caption{\textsc{Comparative Study Results}}
\centering
\tiny
\renewcommand{\arraystretch}{1.2}
\resizebox{\textwidth}{!}{
\begin{tabular}{@{} l ccc ccc ccc ccc @{}}
\toprule
\multirow{2}{*}{Method} & \multicolumn{3}{c}{Average Reward} & \multicolumn{3}{c}{Average Cost} & \multicolumn{3}{c}{Crash Ratio} & \multicolumn{3}{c}{Average Velocity (m/s)} \\
\cmidrule(lr){2-4} \cmidrule(lr){5-7} \cmidrule(lr){8-10} \cmidrule(lr){11-13}
& Low & Medium & High & Low & Medium & High & Low & Medium & High & Low & Medium & High \\
\midrule
SRATRL     & \textbf{48.01} & \textbf{47.34} & \textbf{46.03} & \textbf{0.09} & \textbf{0.11} & \textbf{0.16} & \textbf{0.01} & \textbf{0.01} & \textbf{0.02} & \textbf{21.97} & \textbf{21.80} & \textbf{21.30} \\
SAC-Lag \cite{yang2023safety}      & 45.34 & 43.57 & 43.16 & 1.92 & 1.69 & 1.84 & 0.06 & 0.08 & 0.09 & 21.37 & 20.76 & 20.46 \\
PPO-Lag \cite{stooke2020responsive}  & 42.99 & 42.54 & 42.16 & 0.64 & 0.68 & 0.85 & 0.05 & 0.05 & 0.06 & 20.38 & 20.31 & 20.18 \\
EPGO \cite{peng2022safe} & 45.76  & 44.55  & 44.29 & 0.55  &  0.60 & 0.67 & \textbf{0.01}  & 0.03  & 0.04 & 21.13  & 20.66  & 20.80 \\
FT \cite{julian2021never} &  44.65 & 44.62  & 43.65 & 0.68  & 0.84  & 0.99 &  0.02 &  0.03 & 0.05 &  20.63 & 20.48  & 19.64 \\
\bottomrule
\end{tabular}
}
\vspace{-5mm}
\label{tab-comparative}
\end{table*}

\subsection{Comparative Study}

\subsubsection{Convergence Performance}

Fig.~\ref{fig-comparative} shows the training curves of average reward, average cost, and crash ratio for all compared methods. As shown in Fig.~\ref{fig-comparative}(a), the proposed method achieves rapid reward improvement in the early training stage and reaches the highest reward level after convergence. EPGO shows competitive early-stage performance due to expert guidance, but its final reward and safety metrics remain inferior, suggesting that action-level guidance alone may not be sufficient for target-domain adaptation. SAC-Lag gradually improves its reward with increasing environment interactions, while PPO-Lag converges to a lower reward level, reflecting a more conservative policy with limited traffic efficiency. FT may suffer from behavioral bias caused by source--target mismatch after direct policy initialization, resulting in less effective target-domain adaptation.

From the safety perspective, Fig.~\ref{fig-comparative}(b) shows that the proposed method quickly reduces the average cost and maintains the lowest cost level during most of the training process. In contrast, SAC-Lag exhibits a much higher average cost, suggesting that its reward improvement is accompanied by considerable safety risks. PPO-Lag reduces the safety cost compared with SAC-Lag, but it still shows a clear gap from the proposed method. Since both methods lack teacher guidance, they primarily rely on trial-and-error interactions to learn safe behaviors during target-domain training. The FT method also shows a relatively high average cost, suggesting that teacher policy initialization alone cannot effectively adapt the safety behavior learned in the source domain to the target task. EPGO achieves a lower cost through teacher intervention; however, its cost remains nearly unchanged, suggesting that action-level guidance alone provides limited further improvement in student safety. The crash-ratio curves in Fig.~\ref{fig-comparative}(c) further show that the proposed method maintains a low crash ratio during training, indicating reduced unsafe exploratory interactions and improved training-stage safety. Overall, the proposed method achieves higher reward while maintaining lower average cost and crash ratio, demonstrating improved training-stage safety and a better balance between policy performance and safety.

\subsubsection{Quantitative Evaluation}
Table~\ref{tab-comparative} reports the quantitative evaluation results under low-, medium-, and high-density traffic scenarios. The proposed method achieves the best overall performance across all densities, obtaining the highest average reward and average velocity while maintaining the lowest average cost and crash ratio. Compared with PPO-Lag, the proposed method improves the average reward and average velocity by 10.74\% and 6.90\%, respectively, while reducing the average cost and crash ratio by 83.33\% and 75.00\% on average across the three traffic densities. Although SAC-Lag obtains relatively high velocities in some cases, its higher safety cost and crash ratio indicate that its efficiency is achieved with a greater safety risk.
The proposed method consistently yields lower average cost and crash ratio than all baselines. For example, in the medium-density scenario, the average cost is reduced from 0.68 for PPO-Lag to 0.11, and the crash ratio is reduced from 0.05 to 0.01. These results show that the proposed method can effectively reduce risky interactions and collision events during lane-changing decision-making.
Overall, the proposed method achieves a more favorable balance between traffic efficiency and safety.

\subsection{Evaluation on NGSIM-Based Traffic Data}
We further validate the proposed method using the US Highway 101 (US-101) dataset from the Next Generation Simulation (NGSIM) open data. During evaluation, the recorded trajectories of surrounding vehicles are replayed in the simulation environment, while the ego vehicle is controlled by the evaluated policy.

\begin{table}[t]
\caption{\textsc{Performance Comparison on the NGSIM Dataset}}
\centering
\renewcommand{\arraystretch}{1.2}
\begin{tabular*}{\columnwidth}{@{\extracolsep{\fill}}lcccc}
\toprule
Method & \makecell{Average\\Reward} & \makecell{Average\\Cost} & \makecell{Crash\\Ratio} & \makecell{Average Velocity\\(m/s)} \\
\midrule
SRATRL    & \textbf{47.64} & \textbf{0.73} & \textbf{0.09} & \textbf{21.90} \\
SAC-Lag & 44.89 & 2.31 & 0.39 & 20.78 \\
PPO-Lag & 40.85 & 1.08 & 0.15 & 19.45 \\
EPGO & 44.36 & 0.96 & 0.11 & 19.97 \\
FT & 42.83 & 1.01 & 0.10 & 19.15 \\
\bottomrule
\end{tabular*}
\label{tab-NGSIM}
\vspace{-5mm}
\end{table}
As shown in Table \ref{tab-NGSIM}, from the perspective of traffic safety, the proposed method yields an average cost of 0.73 on the US-101 dataset, lower than the 1.08 achieved by PPO-Lag, while attaining a crash ratio of 0.09, which is also below PPO-Lag’s 0.15. In addition, the proposed method achieves an average reward of 47.64, significantly higher than PPO-Lag’s 40.85, and increases the average velocity from 19.45 m/s to 21.90 m/s, indicating that it can maintain higher traffic efficiency under replay traffic scenarios constructed from NGSIM-recorded vehicle states. These results demonstrate that, in NGSIM-based replayed traffic scenarios, the proposed method achieves superior traffic efficiency while maintaining a high level of safety, thereby further validating its effectiveness under dataset-derived traffic conditions.

\subsection{Ablation Study}
\begin{figure*}[t]
    \centering
    \includegraphics[width=1.0\linewidth]{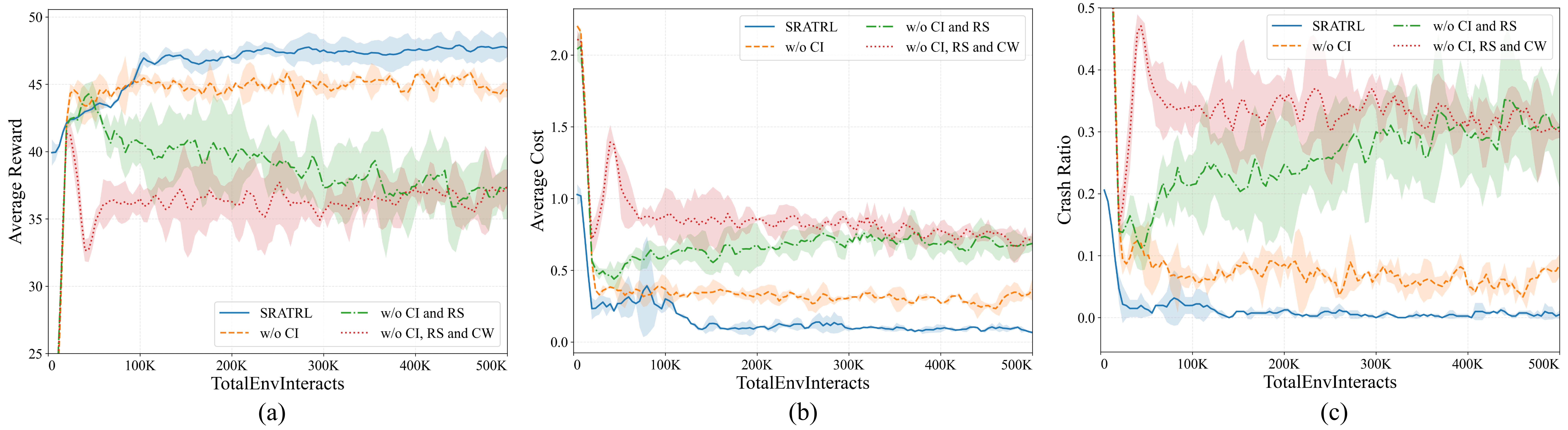}
    \vspace{-6mm}
    \caption{Ablation study. (a) Average reward, (b) Average cost, (c) Crash ratio.  Training curves of average reward, average cost, and crash ratio versus training steps under different ablation configurations in the target-task environment. The results show that the complete method achieves higher returns, lower costs, and a reduced crash ratio compared with all ablated variants, demonstrating the effectiveness of the proposed modules.}
    \label{fig-Ablation}
    \vspace{-2mm}
\end{figure*}

\begin{table*} 
\caption{\textsc{Ablation Study Results}}
\centering
\renewcommand{\arraystretch}{1.2} 
\begin{tabular}
{>{\centering\arraybackslash}p{0.75cm}>{\centering\arraybackslash}p{0.75cm}>{\centering\arraybackslash}p{0.75cm}> 
  {\centering\arraybackslash}p{0.9cm}>{\centering\arraybackslash}p{0.6cm}>{\centering\arraybackslash}p{0.9cm}>
  {\centering\arraybackslash}p{0.9cm}>{\centering\arraybackslash}p{0.6cm}>{\centering\arraybackslash}p{0.9cm}>
  {\centering\arraybackslash}p{0.7cm}>{\centering\arraybackslash}p{0.6cm}>{\centering\arraybackslash}p{0.7cm}>
  {\centering\arraybackslash}p{0.7cm}>{\centering\arraybackslash}p{0.6cm}>{\centering\arraybackslash}p{0.7cm}}
\bottomrule
\multirow{2}{*}{CI} & \multirow{2}{*}{RS} & \multirow{2}{*}{CW} & \multicolumn{3}{c}{Average Reward} & \multicolumn{3}{c}{Average Cost} & \multicolumn{3}{c}{Crash Ratio} & \multicolumn{3}{c}{Average Velocity (m/s)} \\ \cmidrule(lr){4-6} \cmidrule(lr){7-9} \cmidrule(lr){10-12} \cmidrule(lr){13-15} 
&            &            & \multicolumn{1}{c}{Low}   & \multicolumn{1}{c}{Medium} & High  & \multicolumn{1}{c}{Low}  & \multicolumn{1}{c}{Medium} & High & \multicolumn{1}{c}{Low}   & \multicolumn{1}{c}{Medium} & High  & \multicolumn{1}{c}{Low}  & \multicolumn{1}{c}{Medium} & High \\ \hline
 \checkmark & \checkmark & \checkmark & \multicolumn{1}{c}{\textbf{48.01}} & \multicolumn{1}{c}{\textbf{47.34}}  & \textbf{46.03} & \multicolumn{1}{c}{\textbf{0.09}} & \multicolumn{1}{c}{\textbf{0.11}}   & \textbf{0.16} & \multicolumn{1}{c}{\textbf{0.01}} & \multicolumn{1}{c}{\textbf{0.01}}  & \textbf{0.02} & \multicolumn{1}{c}{\textbf{21.97}} & \multicolumn{1}{c}{\textbf{21.80}}   & \textbf{21.30} \\ 
 -          & \checkmark & \checkmark & \multicolumn{1}{c}{47.33} & \multicolumn{1}{c}{46.99}  & 45.82 & \multicolumn{1}{c}{0.23} & \multicolumn{1}{c}{0.28}   & 0.38 & \multicolumn{1}{c}{0.07} & \multicolumn{1}{c}{0.06}  & 0.09 & \multicolumn{1}{c}{19.40} & \multicolumn{1}{c}{19.27}   & 19.77 \\ 
 -          & -          & \checkmark & \multicolumn{1}{c}{42.27} & \multicolumn{1}{c}{41.76}  & 38.87 & \multicolumn{1}{c}{0.44} & \multicolumn{1}{c}{0.44}   & 0.57 & \multicolumn{1}{c}{0.13} & \multicolumn{1}{c}{0.21}  & 0.24 & \multicolumn{1}{c}{19.97} & \multicolumn{1}{c}{19.37}   &19.63 \\ 
 -          & -          & -          & \multicolumn{1}{c}{34.75} & \multicolumn{1}{c}{37.30}  & 36.23 & \multicolumn{1}{c}{0.76} & \multicolumn{1}{c}{0.75}   &0.76 & \multicolumn{1}{c}{0.37} & \multicolumn{1}{c}{0.39}  &0.34 & \multicolumn{1}{c}{20.52} & \multicolumn{1}{c}{20.09}   &20.41 \\ 
\bottomrule
\end{tabular}
\par\vspace{1mm}
\begin{minipage}{\textwidth}
\footnotesize
\textit{Note:} CI denotes safety-triggered closed-loop intervention, 
RS denotes safety-adaptive teacher-guided reward shaping, and 
CW denotes teacher--student policy-compatibility weighting.
\end{minipage}
\label{tab-ablation}
\vspace{-0.5cm}
\end{table*}

To systematically evaluate the contribution of each module, we conduct an ablation study with four configurations: 
(i) the complete method, which includes safety-triggered closed-loop intervention (CI), safety-adaptive teacher-guided reward shaping (RS), and teacher--student policy-compatibility weighting (CW); 
(ii) w/o CI, which removes the adaptive closed-loop regulation in the CI module and replaces it with a fixed teacher-intervention setting; 
(iii) w/o CI and RS, which further removes the safety-adaptive teacher-guided reward shaping term from the value-target construction; 
and (iv) w/o CI, RS and CW,  which removes all the above proposed adaptive modules and retains only the basic fixed teacher-intervention setting.

Fig. \ref{fig-Ablation} (a)--(c) present the ablation results. The complete SRATRL achieves the highest average reward over most of the training process while maintaining the lowest average cost and crash ratio, indicating a better trade-off between efficiency and safety. Its consistently lower safety-related metrics further show that the proposed method can reduce unsafe exploratory behaviors during training. After removing CI, the reward slightly decreases, whereas the safety-related metrics degrade noticeably. Further removing RS leads to a continuous reduction in reward, accompanied by higher cost and crash ratio. The variant without CI, RS, and CW exhibits the weakest performance. 

Table \ref{tab-ablation} reports the ablation results of CI, RS, and CW. When all three modules are removed, the method achieves the lowest rewards and exhibits the highest average costs and crash ratios across different traffic densities. Introducing CW alone significantly improves the reward and reduces both the cost and crash ratio, indicating that policy-compatibility weighting contributes to sample-level optimization. With the further inclusion of RS, the rewards increase to 47.33, 46.99, and 45.82 under low-, medium-, and high-density traffic, respectively, while the crash ratios decrease to 0.07, 0.06, and 0.09. This suggests that teacher-guided reward shaping enhances safety-oriented value learning. The complete SRATRL achieves the best results across all three densities, attaining the highest rewards and average velocities, as well as the lowest costs and crash ratios. These results indicate that each component contributes to the final performance, and their integration effectively improves both policy performance and safety.

\begin{figure*}[t]
    \centering
    \includegraphics[width=1.0\linewidth]{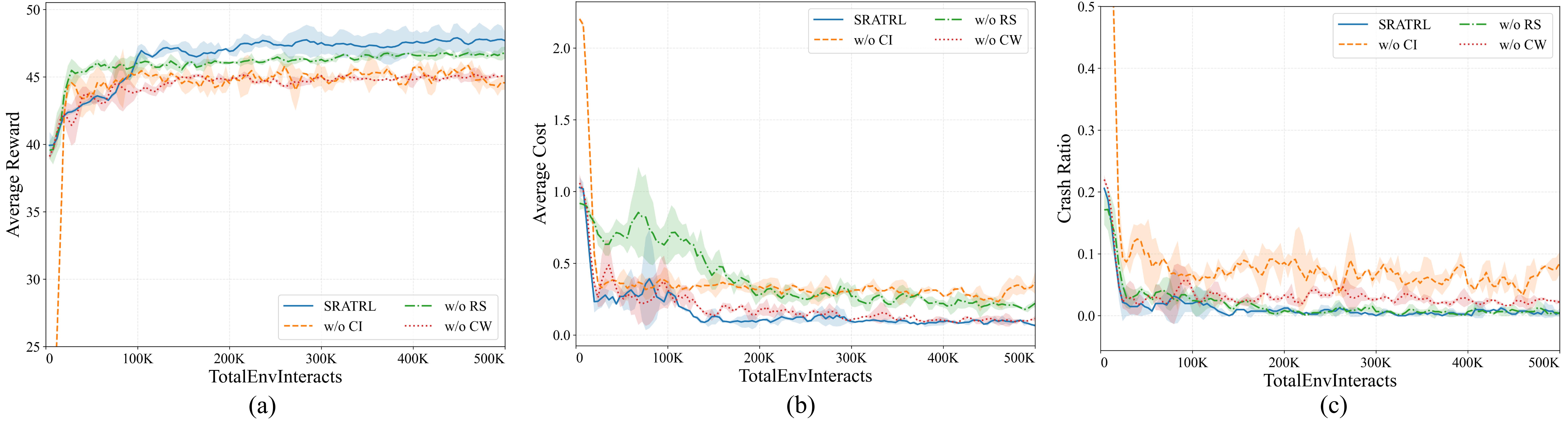}
    \vspace{-6mm}
    \caption{Effectiveness of the main components of SRATRL. The complete SRATRL is compared with three single-module variants, including w/o CI, w/o RS, and w/o CW. (a) Average reward, (b) Average cost, and (c) Crash ratio. The results show that removing any individual component leads to different degrees of performance degradation, while the complete SRATRL achieves higher reward and better safety-related performance.}
    \label{fig-Ablation2}
    \vspace{-1mm}
\end{figure*}
\begin{table*}[t]
\caption{\textsc{Quantitative Evaluation of Main Components}}
\centering
\tiny
\renewcommand{\arraystretch}{1.2}
\resizebox{\textwidth}{!}{
\begin{tabular}{@{} l ccc ccc ccc ccc @{}}
\toprule
\multirow{2}{*}{Method} & \multicolumn{3}{c}{Average Reward} & \multicolumn{3}{c}{Average Cost} & \multicolumn{3}{c}{Crash Ratio} & \multicolumn{3}{c}{Average Velocity (m/s)} \\
\cmidrule(lr){2-4} \cmidrule(lr){5-7} \cmidrule(lr){8-10} \cmidrule(lr){11-13}
& Low & Medium & High & Low & Medium & High & Low & Medium & High & Low & Medium & High \\
\midrule
SRATRL & \multicolumn{1}{c}{\textbf{48.01}} & \multicolumn{1}{c}{\textbf{47.34}}  & \textbf{46.03} & \multicolumn{1}{c}{\textbf{0.09}} & \multicolumn{1}{c}{\textbf{0.11}}   & \textbf{0.16} & \multicolumn{1}{c}{\textbf{0.01}} & \multicolumn{1}{c}{\textbf{0.01}}  & \textbf{0.02} & \multicolumn{1}{c}{\textbf{21.97}} & \multicolumn{1}{c}{\textbf{21.80}}   & \textbf{21.30} \\ 
w/o CI & \multicolumn{1}{c}{47.33} & \multicolumn{1}{c}{46.99}  & 45.82 & \multicolumn{1}{c}{0.23} & \multicolumn{1}{c}{0.28}   & 0.38 & \multicolumn{1}{c}{0.07} & \multicolumn{1}{c}{0.06}  & 0.09 & \multicolumn{1}{c}{19.40} & \multicolumn{1}{c}{19.27}   & 19.77 \\ 
w/o RS &  \multicolumn{1}{c}{47.65}  &  \multicolumn{1}{c}{47.23}  & 45.88  & \multicolumn{1}{c}{0.17}   & \multicolumn{1}{c}{0.19}   & 0.27  &  \multicolumn{1}{c}{0.02}  &  \multicolumn{1}{c}{0.02}  & \multicolumn{1}{c}{0.04}  &  \multicolumn{1}{c}{21.89}  &  \multicolumn{1}{c}{21.44}  &  20.61 \\
w/o CW & \multicolumn{1}{c}{47.42}   &  \multicolumn{1}{c}{47.04}  & 45.64  & \multicolumn{1}{c}{0.13}   & 0.14   & \multicolumn{1}{c}{\textbf{0.16}}  & 0.03   & 0.05   & 0.05  & 20.42   & 20.07   & 20.51  \\
\bottomrule
\end{tabular}
}
\vspace{-3mm}
\label{tab-component}
\end{table*}

\begin{table}[t]
\caption{\centering\textsc{Target-domain evaluation of teacher policy and student policy}}
\centering
\renewcommand{\arraystretch}{1.2}
\begin{tabular*}{\columnwidth}{@{\extracolsep{\fill}}ccccc}
\toprule
Policy & \makecell{Average\\Reward} & \makecell{Average\\Cost} & \makecell{Crash\\Ratio} & \makecell{Average Velocity\\(m/s)} \\
\midrule
Teacher    & 30.37 & 0.25 & \textbf{0.01} & 18.82 \\
Initial Student & 16.29 & 1.96 & 0.75 & 20.26 \\
Final Student & \textbf{47.34} & \textbf{0.11} & \textbf{0.01} & \textbf{21.80} \\
\bottomrule
\par\vspace{0.1mm}
\end{tabular*}
\par\vspace{-1mm}
\noindent
\begin{minipage}{\columnwidth}
\scriptsize
\emph{Note:} Initial Student denotes the randomly initialized student policy before target-domain training. Final Student denotes the learned student policy after training with the proposed teacher-guided safe transfer framework.
\end{minipage}
\label{tab-teacher}
\vspace{-2mm}
\end{table}

\subsection{Effectiveness of Main Components}
To further examine the effectiveness of the main components, this subsection provides both component-level comparisons and teacher-policy evaluation. 
For component-level analysis, the complete SRATRL is compared with three variants, namely w/o CI, w/o RS, and w/o CW. 
Unlike the progressive ablation study, only one component is removed each time, allowing the individual contribution of CI, RS, and CW to be more clearly observed, as shown in Fig. \ref{fig-Ablation2}. 
In addition, the standalone teacher policy is evaluated in the target environment and compared with the initial and final student policies, aiming to verify whether the source-domain teacher can serve as a useful behavioral prior during target-domain adaptation.

\subsubsection{Effectiveness of Safety-Triggered Closed-Loop Intervention}
The safety-triggered closed-loop intervention module is evaluated by comparing SRATRL with w/o CI, where the proposed intervention strategy is replaced by a fixed intervention setting. As shown in Fig. \ref{fig-Ablation2}, w/o CI achieves a lower average reward than SRATRL after convergence, while its average cost and crash ratio remain consistently higher. This indicates that adaptively regulating teacher intervention is beneficial for coordinating teacher guidance and student exploration during training. The proposed CI module introduces teacher intervention according to the safety condition of the current interaction and gradually reduces unnecessary teacher dependence. As a result, the student policy can benefit from teacher guidance in risky states while still retaining sufficient autonomy for policy improvement. Table \ref{tab-component} further confirms the safety benefit of CI. Under medium-density traffic, removing CI increases the average cost from 0.11 to 0.28 and the crash ratio from 0.01 to 0.06, while reducing the average reward from 47.34 to 46.99. Therefore, CI contributes to safer exploration and improves the final policy performance.

\subsubsection{Effectiveness of Safety-Adaptive Value Shaping}
The effectiveness of safety-adaptive value shaping is verified by comparing SRATRL with w/o RS, where the teacher-guided reward shaping term is removed from the value-target construction. From Fig. \ref{fig-Ablation2}, w/o RS can still learn a feasible policy, but its reward is lower than that of SRATRL in the later training stage. Meanwhile, w/o RS exhibits higher average cost during a considerable part of training, indicating that removing the teacher-guided shaping signal weakens the safety-oriented value guidance. By introducing the adaptive reward shaping term, SRATRL incorporates the teacher policy as a behavioral prior into value learning. This allows the critic to assign a higher value to actions that are more consistent with teacher guidance, especially during early and middle training. A similar effect can be observed in the quantitative evaluation. Without RS, the average cost under medium-density traffic rises from 0.11 to 0.19, accompanied by a slight decrease in average reward from 47.34 to 47.23. This result is consistent with the role of RS in providing safety-oriented value guidance. Consequently, RS helps improve both task reward and safety-related performance.
\begin{figure}[t]
    \centering
    \includegraphics[width=0.7\linewidth]{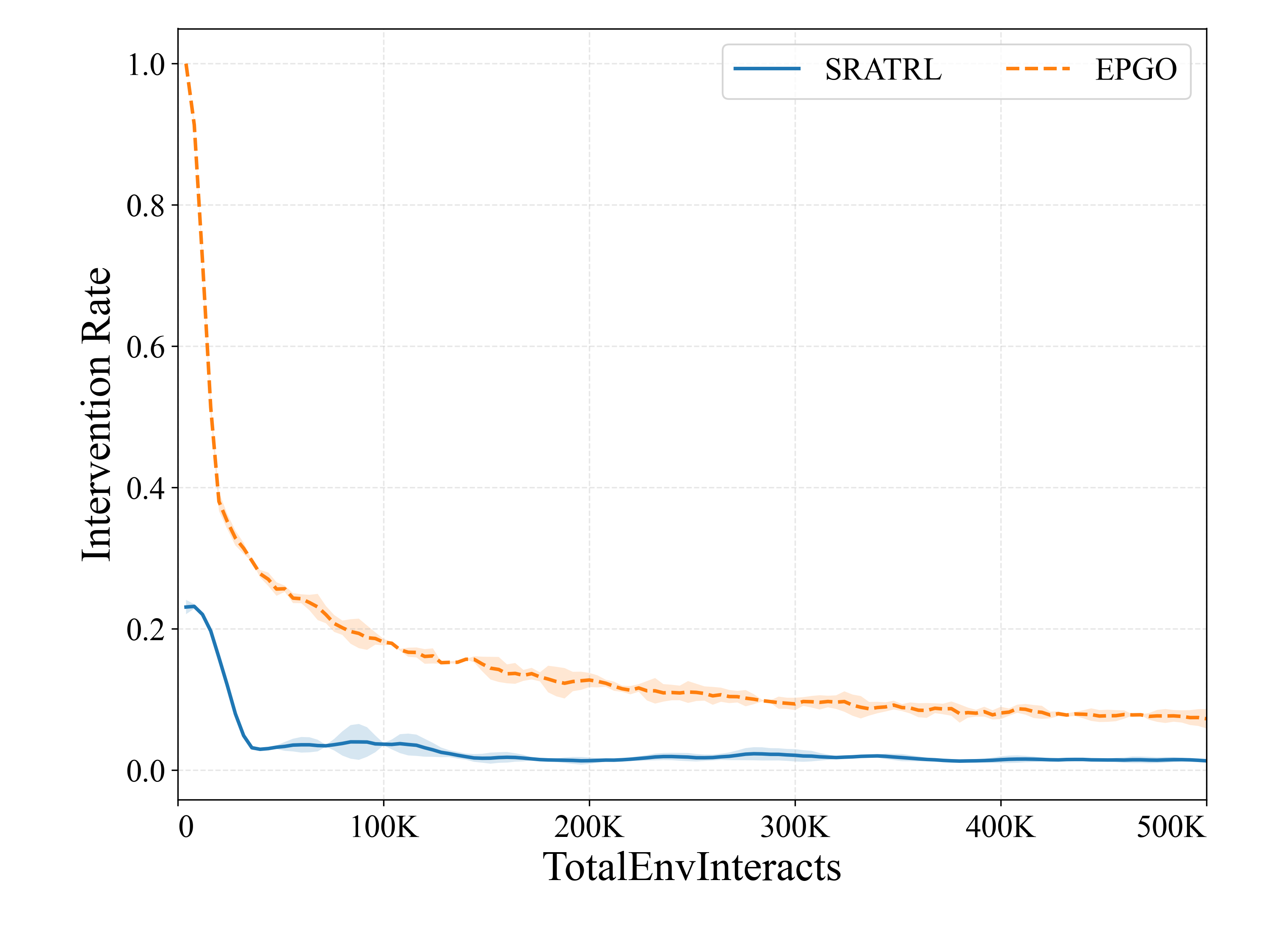}
    \vspace{-3mm}
    \captionsetup{justification=raggedright, singlelinecheck=false}
    \caption{Comparison of teacher intervention rates between SRATRL and EPGO during student training.}
    \label{fig-InterventionRate}
    \vspace{-2mm}
\end{figure}

\subsubsection{Effectiveness of Policy-Compatibility-Based Optimization Strategies}
The policy-compatibility-based optimization strategy is evaluated by comparing SRATRL with w/o CW, where the compatibility weighting mechanism is removed during optimization. As shown in Fig. \ref{fig-Ablation2}, w/o CW obtains a lower average reward than the complete method and shows slightly higher average cost and crash ratio. This suggests that directly using teacher-guided samples without considering the compatibility between the teacher and student policies may reduce the effectiveness of sample utilization. The proposed CW module assigns optimization weights according to the relative likelihood between the student and teacher policies, thereby emphasizing samples that are more compatible with the current student policy. This design helps reduce the influence of less compatible teacher-guided samples and improves the effectiveness of policy optimization. The results in Table \ref{tab-component} also show that removing CW reduces the average reward from 47.34 to 47.04 under medium-density traffic, while increasing the crash ratio from 0.01 to 0.05. This suggests that compatibility-based weighting improves the effectiveness of teacher-guided samples during policy optimization. 

\begin{figure*}[t]
    \centering
    \includegraphics[width=1.0\linewidth]{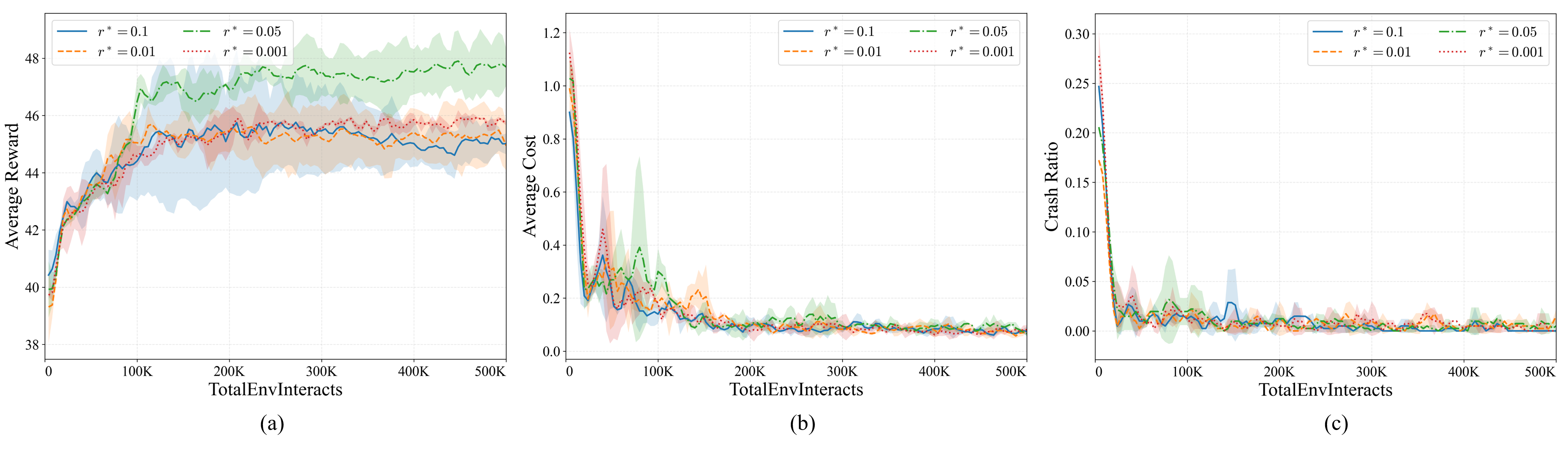}
    \vspace{-7mm}
    \caption{Sensitivity analysis of the reference intervention rate $r^{*}$. (a) Average reward, (b) Average cost. Training curves of average reward (left) and average cost (right) under different reference intervention-rate settings ($r^{*}\in\{0.1,0.05,0.01,0.001\}$), illustrating their impact on training behavior. }
    \label{fig-Sensitivity}
    \vspace{-3mm}
\end{figure*}

\subsubsection{Teacher Policy Evaluation and Intervention Analysis}
In addition to the component-level analysis, we further examine the standalone teacher policy in the target environment and present the evolution of the teacher intervention rate during student training. This evaluation is not intended to regard the teacher as a complete target-task learning baseline, but to verify whether the source-domain policy can provide a low-risk behavioral prior for student training.

As shown in Table~\ref{tab-teacher}, under the medium-density traffic setting, the teacher policy achieves substantially lower average cost and crash ratio than the untrained student policy, indicating that it can provide safety-oriented guidance during early target-domain exploration. Nevertheless, due to the source--target task mismatch, the standalone teacher policy is not necessarily optimal for the target task. Compared with the teacher policy, the final student policy obtains higher average reward and velocity while further reducing the average cost. This suggests that the student policy does not simply imitate the teacher, but improves its task performance through target-domain adaptation. Therefore, in the proposed framework, the teacher policy is used as a source-domain behavioral prior rather than a complete controller for the target task.

To further examine the reliance on teacher intervention during training, Fig.~\ref{fig-InterventionRate} compares the intervention rates of SRATRL and EPGO. Both methods gradually reduce teacher intervention as training proceeds, while SRATRL maintains a substantially lower intervention rate and decreases more rapidly during the early training stage. This indicates that the proposed closed-loop intervention mechanism can reduce the student’s reliance on teacher intervention while retaining teacher guidance during training. In contrast, EPGO relies on more frequent teacher intervention throughout training. These results further support the role of the teacher policy in SRATRL as a temporary source-domain behavioral prior rather than a persistent controller.

\subsection{Sensitivity Analysis}
To evaluate the impact of the reference intervention rate, we conduct a sensitivity analysis with 
$r^{*}\in\{0.1,0.05,0.01,0.001\}$. The corresponding average reward and average cost curves are shown in Fig.~\ref{fig-Sensitivity}. Overall, all settings effectively reduce the average cost during training and eventually converge to a low cost level. However, different reference intervention rates lead to different reward performance. Among the tested settings, $r^{*}=0.05$ achieves the highest average reward after convergence while maintaining a low average cost.  These results indicate that the proposed method exhibits stable performance across different reference intervention-rate settings, with $r^{*}=0.05$ providing the most favorable overall performance in the evaluated scenarios.

\begin{figure*}[t]
    \centering
    \includegraphics[width=1.0\linewidth]{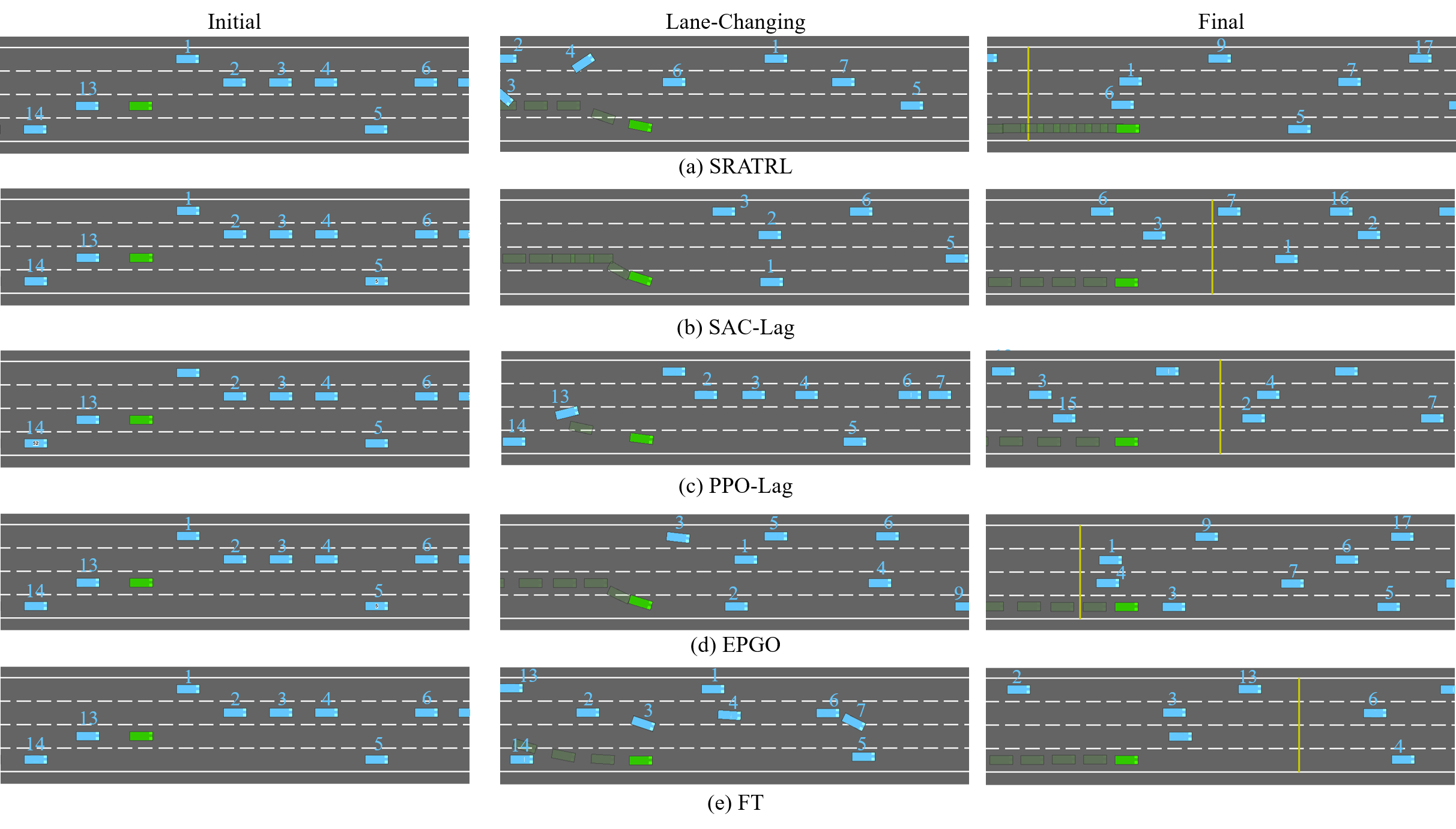}
    \vspace{-7mm}
    \caption{Visualization of representative lane-changing behaviors under different methods. For each method, the three snapshots correspond to the initial state, the lane-changing process, and the final state, while the yellow vertical line indicates a fixed longitudinal reference position for comparing driving progress.}
    \label{fig-Comparative-Visualization}
    \vspace{-2mm}
\end{figure*}
\subsection{Visualization Results}
To qualitatively compare the driving behaviors of different methods, Fig.~\ref{fig-Comparative-Visualization} visualizes representative lane-changing processes of the compared methods. For each method, the three snapshots correspond to the initial state, the lane-changing process, and the final state, respectively. The yellow vertical line marks a fixed longitudinal reference position from the starting point and is used to compare the longitudinal progress of different methods. As shown in Fig.~\ref{fig-Comparative-Visualization}, all methods exhibit different lane-changing trajectories and longitudinal progress. SRATRL completes the lane-changing maneuver smoothly and, at the final state, shows the greatest longitudinal progress, as indicated by the relative position of the fixed yellow reference line.  Overall, the visualization provides an intuitive comparison of the lane-changing behaviors and longitudinal driving efficiency of different methods, further illustrating the driving performance of SRATRL.

\vspace{-2mm}
\subsection{Discussion}
The proposed framework aims to balance safety and driving efficiency while improving safety throughout target-domain adaptation. The teacher policy is introduced as a source-domain behavioral prior to guide early-stage exploration, rather than as an oracle safety expert. Its guiding influence progressively decays as the student policy improves. The intervention strategy provides conservative guidance after high-cost events, which helps generate coherent teacher-guided trajectories during early training. Nevertheless, more fine-grained step-level recovery strategies may further improve the autonomy of the student policy. In addition, the current cost design provides a tractable interaction-risk measure for highway lane-changing scenarios. Future work will extend the safety formulation to include richer front-rear and lateral interaction risks, and investigate more flexible teacher-student action fusion mechanisms.

\vspace{-2mm}
\section{Conclusion}
\label{section9}
In this paper, we propose SRATRL, a safety-regulated adaptive transfer reinforcement learning framework for autonomous highway lane changing. 
The framework uses a source-domain teacher policy as a behavioral prior for target-domain adaptation and regulates its influence through three coordinated components. 
The safety-triggered closed-loop intervention module provides adaptive teacher guidance; the safety-adaptive teacher-guided reward shaping module introduces teacher guidance into value learning; and the teacher-student policy-compatibility weighting module adjusts the use of teacher-guided samples according to their compatibility with the student policy.
In addition, return-deviation bounds are provided to characterize the effect of teacher intervention on the mixed behavior policy. Experimental results show that the proposed method improves training-stage safety during student training in the target domain, achieves a favorable balance between traffic efficiency and safety, and outperforms the considered baseline methods in autonomous highway lane-changing tasks. Future work will extend the framework to more diverse traffic scenarios and investigate more flexible teacher-student action fusion mechanisms.

\bibliographystyle{IEEEtran}
\bibliography{IEEE.bib}

@ARTICLE{Ma2024TITS,
  author={Ma, Zhenyu and Liu, Xinyi and Huang, Yanjun},
  journal={IEEE Trans. Intell. Transp. Syst.	}, 
  title={Unsupervised Reinforcement Learning for Multi-Task Autonomous Driving: Expanding Skills and Cultivating Curiosity}, 
  year={2024},
  volume={25},
  number={10},
  pages={14209-14219},
  doi={10.1109/TITS.2024.3400224}}

@inproceedings{Krasowski2020SafeLaneChange,
  author    = {Hanna Krasowski and Xiao Wang and Matthias Althoff},
  title     = {Safe Reinforcement Learning for Autonomous Lane Changing Using Set-Based Prediction},
  booktitle = {IEEE Int. Conf. Intell. Transp. Syst. (ITSC)},
  pages     = {1--7},
  year      = {2020},
  doi       = {10.1109/ITSC45102.2020.9294259}
}

@inproceedings{Zhang2024PerspectiveO2O,
  author  = {Yinmin Zhang and Jie Liu and Chuming Li and Yazhe Niu and Yaodong Yang and Yu Liu and Wanli Ouyang},
  title   = {A Perspective of Q-value Estimation on Offline-to-Online Reinforcement Learning},
 booktitle={Proc. AAAI Conf. Artif. Intell.},
  volume={38},
  number={15},
  pages={16908--16916},
  year={2024}
}

@article{Hsu2023SimToLabToReal,
author = {Kai-Chieh Hsu and Allen Z. Ren and Duy Phuong Nguyen and Anirudha Majumdar and Jaime F. Fisac},
title = {Sim-to-Lab-to-Real: Safe Reinforcement Learning with Shielding and Generalization Guarantees},
journal = {Artif. Intell.},
volume = {314},
pages = {103811},
year = {2023}
}

@article{Zhou2025S2CD,
  author  = {Rongliang Zhou and Jiakun Huang and Mingjun Li and Hepeng Li and Haotian Cao and Xiaolin Song},
  title   = {Knowledge Transfer from Simple to Complex: A Safe and Efficient Reinforcement Learning Framework for Autonomous Driving Decision-Making},
  journal = {Adv. Eng. Inform.},
  year    = {2025},
  doi     = {10.1016/j.aei.2025.103188}
}

@article{Moni2025SSDT,
  author  = {R{\'o}bert Moni and B{\'a}lint Gyires-T{\'o}th},
  title   = {Self-supervised Domain Transfer for Reinforcement Learning-Based Autonomous Driving Agent},
  journal = {Expert Syst. Appl.},
  volume  = {284},
  pages   = {127809},
  year    = {2025},
  doi     = {10.1016/j.eswa.2025.127809}
}

@inproceedings{You2022CAT,
  author    = {Haoxiang You and Ruihai Dong and Yuejie Chi and Yan Zhu},
  title     = {Cross-domain Adaptive Transfer Reinforcement Learning based on State-Action Correspondence},
  booktitle = {Proc. Conf. Uncertainty Artif. Intell.},
  volume    = {180},
  pages     = {1640--1652},
  year      = {2022}
}

@article{Zhang2024SafetyControlTL,
  author  = {Quanqi Zhang and Chengwei Wu and Haoyu Tian and Yabin Gao and Weiran Yao and Ligang Wu},
  title   = {Safety Reinforcement Learning Control via Transfer Learning},
  journal = {Automatica},
  volume  = {166},
  pages   = {111714},
  year    = {2024},
  doi     = {10.1016/j.automatica.2024.111714}
}

@article{Liang2019FedTransferRL,
author = {Xinle Liang and Yang Liu and Tianjian Chen and Ming Liu and Qiang Yang},
title = {Federated Transfer Reinforcement Learning for Autonomous Driving},
journal = {arXiv preprint arXiv:1910.06001},
year = {2019},
doi = {10.48550/arXiv.1910.06001}
}

@article{Lu2024SceTL,
author = {Hongliang Lu and Chao Lu and Haoyang Wang and Jianwei Gong and Meixin Zhu and Hai Yang},
title = {Scenario-level knowledge transfer for motion planning of autonomous driving via successor representation},
journal = {Transp. Res. Pt. C-Emerg. Technol.},
volume = {168},
pages = {104899},
year = {2024},
doi = {10.1016/j.trc.2024.104899}
}

@ARTICLE{Shu2022TVT,
  author={Shu, Hong and Liu, Teng and Mu, Xingyu and Cao, Dongpu},
  journal={IEEE Trans. Veh. Technol.}, 
  title={Driving Tasks Transfer Using Deep Reinforcement Learning for Decision-Making of Autonomous Vehicles in Unsignalized Intersection}, 
  year={2022},
  volume={71},
  number={1},
  pages={41-52},
  doi={10.1109/TVT.2021.3121985}}

@article{Kiran2022DRLSurvey,
  author    = {B. Ravi Kiran and Ibrahim Sobh and Victor Talpaert and Patrick Mannion and Ahmad A. Al Sallab and Senthil Yogamani and Patrick P{\'e}rez},
  title     = {Deep Reinforcement Learning for Autonomous Driving: A Survey},
  journal   = {IEEE Trans. Intell. Transp. Syst. },
  volume    = {23},
  number    = {6},
  pages     = {4909--4926},
  year      = {2022},
  doi       = {10.1109/TITS.2021.3054625}
}

@inproceedings{Xu2018ZeroShotTransfer,
  author    = {Zhuo Xu and Chen Tang and Masayoshi Tomizuka},
  title     = {Zero-shot Deep Reinforcement Learning Driving Policy Transfer for Autonomous Vehicles based on Robust Control},
  booktitle = {Proc. 21st IEEE Intell. Transp. Syst. Conf. (ITSC)},
  pages     = {2865--2871},
  year      = {2018},
  organization = {IEEE},
  doi       = {10.1109/ITSC.2018.8569612}
}

@article{xue2023guarded,
  title   = {Guarded Policy Optimization with Imperfect Online Demonstrations},
  author  = {Zhenghai Xue and Zhenghao Peng and Quanyi Li and Zhihan Liu and Bolei Zhou},
  journal = {Proc. 11th Int. Conf. Learn. Represent. (ICLR)},
  year    = {2023},
}

@inproceedings{kakade2002approximately,
  title={Approximately optimal approximate reinforcement learning},
  author={Kakade, Sham and Langford, John},
  booktitle={Proc. 19th Int. Conf. Mach. Learn.},
  pages={267--274},
  year={2002}
}

@book{folland1999real,
  title={Real analysis: modern techniques and their applications},
  author={Folland, Gerald B},
  year={1999},
  publisher={John Wiley \& Sons}
}

@book{cover2006elements,
  title={Elements of Information Theory},
  author={Cover, Thomas M. and Thomas, Joy A.},
  edition={2},
  year={2006},
  publisher={Wiley}
}

@inproceedings{yang2023reinforcement,
  title={Reinforcement Learning by Guided Safe Exploration},
  author={Yang, Qisong and Sim{\~a}o, Thiago D and Jansen, Nils and Tindemans, Simon H and Spaan, Matthijs TJ},
  booktitle={Proc. 26th Eur. Conf. Artif. Intell.},
  pages={2858--2865},
  year={2023}
}

@inproceedings{Alshiekh2018Shielding,
  title={Safe reinforcement learning via shielding},
  author={Alshiekh, Mohammed and Bloem, Roderick and Ehlers, R{\"u}diger and K{\"o}nighofer, Bettina and Niekum, Scott and Topcu, Ufuk},
  booktitle={Proc. AAAI Conf. Artif. Intell.},
  volume={32},
  number={1},
  year={2018}
}

@inproceedings{Ilhan2021AdviceImitation,
author = {Ilhan, Ercument and Gow, Jeremy and Perez Liebana, Diego},
title = {Action Advising with Advice Imitation in Deep Reinforcement Learning},
year = {2021},
booktitle = {Proc. 20th Int. Conf. Auton. Agents Multiagent Syst.},
pages = {629–637},
numpages = {9},
}

@article{Edouard2018Highway,
  author = {Leurent, Edouard},
  title = {An Environment for Autonomous Driving Decision-Making},
  year = {2018},
  publisher = {GitHub},
  journal = {GitHub repository},
  howpublished = {\url{https://github.com/eleurent/highway-env}},
}

@INPROCEEDINGS{Polack2017IV,
  author={Polack, Philip and Altché, Florent and d'Andréa-Novel, Brigitte and de La Fortelle, Arnaud},
  booktitle={Proc. IEEE Intell. Veh. Symp. (IV 2017)}, 
  title={The kinematic bicycle model: A consistent model for planning feasible trajectories for autonomous vehicles?}, 
  year={2017},
  volume={},
  number={},
  pages={812-818},
  doi={10.1109/IVS.2017.7995816}}

@article{treiber2000congested,
  title={Congested traffic states in empirical observations and microscopic simulations},
  author={Treiber, Martin and Hennecke, Ansgar and Helbing, Dirk},
  journal={Phys. Rev. E},
  volume={62},
  number={2},
  pages={1805},
  year={2000}
}

@INPROCEEDINGS{2001Preferred,
  author={Ayres, T.J. and Li, L. and Schleuning, D. and Young, D.},
  booktitle={Proc. 4th IEEE Intell. Transp. Syst. Conf. (ITSC)}, 
  title={Preferred time-headway of highway drivers}, 
  year={2001},
  volume={},
  number={},
  pages={826-829},
}

@article{Vecerik2017LeveragingDemonstrations,
  title={Leveraging demonstrations for deep reinforcement learning on robotics problems with sparse rewards},
  author={Vecerik, Mel and Hester, Todd and Scholz, Jonathan and Wang, Fumin and Pietquin, Olivier and Piot, Bilal and Heess, Nicolas and Roth{\"o}rl, Thomas and Lampe, Thomas and Riedmiller, Martin},
  journal={arXiv preprint arXiv:1707.08817},
  year={2017},
  doi = {10.48550/arXiv.1707.08817}
}

@inproceedings{Kang2018POfD,
  title     = {Policy Optimization with Demonstrations},
  author    = {Kang, Bingyi and Jie, Zequn and Feng, Jiashi},
  booktitle = {Proc. 35th Int. Conf. Mach. Learn.},
  volume    = {80},
  pages     = {2474--2483},
  year      = {2018}
}

@article{Parisotto2016ActorMimic,
  title={Actor-mimic: Deep multitask and transfer reinforcement learning},
  author={Parisotto, Emilio and Ba, Jimmy Lei and Salakhutdinov, Ruslan},
  journal={arXiv preprint arXiv:1511.06342},
  year={2015},
  doi={10.48550/arXiv.1511.06342}
}

@inproceedings{Yin2017HierarchicalExperienceReplay,
  title={Knowledge transfer for deep reinforcement learning with hierarchical experience replay},
  author={Yin, Haiyan and Pan, Sinno},
  booktitle={Proc. AAAI Conf. Artif. Intell.},
  volume={31},
  number={1},
  year={2017}
}

@inproceedings{Wang2017ConfidenceBasedDemonstrations,
author = {Wang, Zhaodong and Taylor, Matthew E.},
title = {Improving reinforcement learning with confidence-based demonstrations},
year = {2017},
isbn = {9780999241103},
booktitle = {Proc. 26th Int. Joint Conf. Artif. Intell.},
pages = {3027–3033},
numpages = {7},
location = {Melbourne, Australia},
}

@inproceedings{Anand2021StateCategorization,
  title={An enhanced advising model in teacher-student framework using state categorization},
  author={Anand, Daksh and Gupta, Vaibhav and Paruchuri, Praveen and Ravindran, Balaraman},
  booktitle={Proc. AAAI Conf. Artif. Intell.},
  volume={35},
  number={8},
  pages={6653--6660},
  year={2021}
}

@inproceedings{Guo2023CRSfD,
  title={Reinforcement learning with demonstrations from mismatched task under sparse reward},
  author={Guo, Yanjiang and Gao, Jingyue and Wu, Zheng and Shi, Chengming and Chen, Jianyu},
  booktitle={Proc. Conf. Robot Learn.},
  pages={1146--1156},
  year={2023}
}

@inproceedings{Filos2021SuccessorDemo,
  title={Psiphi-learning: Reinforcement learning with demonstrations using successor features and inverse temporal difference learning},
  author={Filos, Angelos and Lyle, Clare and Gal, Yarin and Levine, Sergey and Jaques, Natasha and Farquhar, Gregory},
  booktitle={Proc. 35th Int. Conf. Mach. Learn.},
  pages={3305--3317},
  year={2021}
}

@inproceedings{stooke2020responsive,
  title={Responsive safety in reinforcement learning by pid lagrangian methods},
  author={Stooke, Adam and Achiam, Joshua and Abbeel, Pieter},
  booktitle={Proc. 37th Int. Conf. Mach. Learn.},
  pages={9133--9143},
  year={2020}
}

@article{borkar2005actor,
  title={An actor-critic algorithm for constrained Markov decision processes},
  author={Borkar, Vivek S},
  journal={Syst. Control Lett.},
  volume={54},
  number={3},
  pages={207--213},
  year={2005}
}

@article{Huang2024HAIMDRL,
  author  = {Huang, Zilin and Sheng, Zihao and Ma, Chengyuan and Chen, Sikai},
  title   = {Human as AI Mentor: Enhanced Human-in-the-Loop Reinforcement Learning for Safe and Efficient Autonomous Driving},
  journal = {Commun. Transp. Res.},
  year    = {2024},
  pages   = {100127},
  doi     = {10.1016/j.commtr.2024.100127}
}

@inproceedings{peng2022safe,
  title={Safe driving via expert guided policy optimization},
  author={Peng, Zhenghao and Li, Quanyi and Liu, Chunxiao and Zhou, Bolei},
  booktitle={Proc. Conf. Robot Learn.},
  pages={1554--1563},
  year={2022},
  organization={PMLR}
}

@inproceedings{julian2021never,
  title={Never Stop Learning: The Effectiveness of Fine-Tuning in Robotic Reinforcement Learning},
  author={Julian, Ryan and Swanson, Benjamin and Sukhatme, Gaurav and Levine, Sergey and Finn, Chelsea and Hausman, Karol},
  booktitle={Proc. Conf. Robot Learn.},
  pages={2120--2136},
  year={2021},
  organization={PMLR}
}

@inproceedings{Mueller2018DrivingPolicyTransfer,
  title={Driving Policy Transfer via Modularity and Abstraction},
  author={Mueller, Matthias and Dosovitskiy, Alexey and Ghanem, Bernard and Koltun, Vladlen},
  booktitle={Proc. Conf. Robot Learn.},
  pages={1--15},
  year={2018},
  organization={PMLR}
}

@inproceedings{Xu2018ZeroShotDrivingTransfer,
  title={Zero-shot deep reinforcement learning driving policy transfer for autonomous vehicles based on robust control},
  author={Xu, Zhuo and Tang, Chen and Tomizuka, Masayoshi},
  booktitle={Proc. 21st IEEE Int. Conf. Intell. Transp. Syst. (ITSC)},
  pages={2865--2871},
  year={2018},
  organization={IEEE}
}

@inproceedings{Campbell2023IntrospectiveActionAdvising,
  title={Introspective action advising for interpretable transfer learning},
  author={Campbell, Joseph and Guo, Yue and Xie, Fiona and Stepputtis, Simon and Sycara, Katia},
  booktitle={Proc. Conf. Lifelong Learn. Agents (CoLLAs)},
  pages={1072--1090},
  year={2023},
  organization={PMLR}
}

@inproceedings{Achiam2017CPO,
  title={Constrained policy optimization},
  author={Achiam, Joshua and Held, David and Tamar, Aviv and Abbeel, Pieter},
  booktitle={Proc. 34th Int. Conf. Mach. Learn.},
  pages={22--31},
  year={2017},
  organization={Pmlr}
}

@article{Chow2015CVaR,
  title={Risk-sensitive and robust decision-making: a cvar optimization approach},
  author={Chow, Yinlam and Tamar, Aviv and Mannor, Shie and Pavone, Marco},
  journal={Adv. Neural Inform. process. syst.},
  volume={28},
  year={2015}
}

@article{Dalal2018SafeExploration,
  title={Safe exploration in continuous action spaces},
  author={Dalal, Gal and Dvijotham, Krishnamurthy and Vecerik, Matej and Hester, Todd and Paduraru, Cosmin and Tassa, Yuval},
  journal={arXiv preprint arXiv:1801.08757},
  year={2018}
}

@article{Thananjeyan2021RecoveryRL,
  title={Recovery rl: Safe reinforcement learning with learned recovery zones},
  author={Thananjeyan, Brijen and Balakrishna, Ashwin and Nair, Suraj and Luo, Michael and Srinivasan, Krishnan and Hwang, Minho and Gonzalez, Joseph E and Ibarz, Julian and Finn, Chelsea and Goldberg, Ken},
  journal={IEEE Rob. Autom. Lett.},
  volume={6},
  number={3},
  pages={4915--4922},
  year={2021},
  publisher={IEEE}
}

@inproceedings{Gupta2019AdviceReplay,
  title={Advice replay approach for richer knowledge transfer in teacher student framework},
  author={Gupta, Vaibhav and Anand, Daksh and Paruchuri, Praveen and Ravindran, Balaraman},
  booktitle={Proc. 18th Int. Conf. Auton. Agents Multiagent Syst. (AAMAS)},
  pages={1997--1999},
  year={2019}
}

@article{yang2023safety,
  title={Safety-constrained reinforcement learning with a distributional safety critic},
  author={Yang, Qisong and Sim{\~a}o, Thiago D and Tindemans, Simon H and Spaan, Matthijs TJ},
  journal={Mach. Learn.},
  volume={112},
  number={3},
  pages={859--887},
  year={2023},
  publisher={Springer}
}

@inproceedings{yang2021efficient,
  title={Efficient deep reinforcement learning via adaptive policy transfer},
  author={Yang, Tianpei and Hao, Jianye and Meng, Zhaopeng and Zhang, Zongzhang and Hu, Yujing and Chen, Yingfeng and Fan, Changjie and Wang, Weixun and Liu, Wulong and Wang, Zhaodong and others},
  booktitle={Proc. 29th Int. Joint Conf. Artif. Intell. (IJCAI)},
  pages={3094--3100},
  year={2021}
}

@article{li2026risk,
  title={Risk-Constrained On-Ramp Merging via Safety-Augmented Reinforcement Learning and Model Predictive Control},
  author={Li, Yang and Li, Jian and Huang, Wenjie and Yang, Qisong and Qin, Hongmao and Jiang, Xiaolong and Bian, Yougang and Hu, Manjiang and Hu, Yingbai},
  journal={IEEE Internet Things J.},
  year={2026},
  publisher={IEEE}
}

@Article{machines14060605,
AUTHOR = {Teng, Jingjia and Huang, Wenjie and Yuan, Shijie and Hu, Manjiang and Qin, Hongmao and Li, Yang and Bian, Yougang and Li, Bai},
TITLE = {Adaptive Constraint Regulation for Human Preference-Aware Safe Reinforcement Learning of On-Ramp Merging},
JOURNAL = {Machines},
VOLUME = {14},
YEAR = {2026},
NUMBER = {6},
ARTICLE-NUMBER = {605},
ISSN = {2075-1702}
}

\end{document}